\documentclass{article}

\usepackage{microtype}
\usepackage{graphicx}
\usepackage{subfigure}
\usepackage{booktabs} 
\usepackage{pifont} 
\usepackage{amsmath} 
\usepackage{thm-restate}
\usepackage{hyperref}

\usepackage[accepted]{icml2025}

\usepackage{amsthm}
\usepackage{amsmath}
\usepackage{amssymb}
\usepackage{mathtools}
\usepackage{graphicx}
\usepackage{adjustbox}
\usepackage[capitalize,noabbrev]{cleveref}

\theoremstyle{plain}
\newtheorem{theorem}{Theorem}[section]

\newtheorem{lemma}[theorem]{Lemma}

\theoremstyle{definition}

\theoremstyle{remark}
\newtheorem{remark}[theorem]{Remark}

\usepackage[textsize=tiny]{todonotes}

\icmltitlerunning{Submission and Formatting Instructions for ICML 2025}

\usepackage{scalerel} 

\definecolor{uclablue}{rgb}{0.15, 0.45, 0.68}
\hypersetup{
    breaklinks,
    citecolor=uclablue,
    colorlinks=true,
    linkcolor=uclablue
}

\usepackage{multirow}

\usepackage{enumitem}
\setenumerate[1]{itemsep=0pt,partopsep=0pt,parsep=\parskip,topsep=5pt}
\setitemize[1]{itemsep=0pt,partopsep=0pt,parsep=\parskip,topsep=5pt}
\setdescription{itemsep=0pt,partopsep=0pt,parsep=\parskip,topsep=5pt}

\usepackage{xspace}
\usepackage{ulem}
\newcommand{\model}{FloE\xspace}
\newtheorem{question}{\bf Question}

\newtheorem{observation}{\bf Observation}

\newcommand{\lz}[1]{%
{\textcolor{black}{#1}}%
}
\newcommand{\zl}[1]{%
{\textcolor{black}{#1}}%
}
\newcommand{\zj}[1]{%
{\textcolor{black}{#1}}%
}
\newcommand{\zyx}[1]{%
{\textcolor{black}{#1}}%
}
\begin{document}

\twocolumn[
\icmltitle{
\scalerel*{\includegraphics{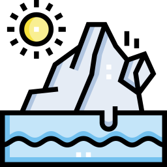}}{{\rule{2ex}{2ex}}}
\model: On-the-Fly MoE Inference on Memory-constrained GPU
}


\icmlsetsymbol{equal}{*}

\begin{icmlauthorlist}
\icmlauthor{Yuxin Zhou}{equal,yyy,comp}
\icmlauthor{Zheng Li}{equal,yyy,comp}
\icmlauthor{Jun Zhang}{yyy,comp}
\icmlauthor{Jue Wang}{yyy}
\icmlauthor{Yiping Wang}{uw}
\icmlauthor{Zhongle Xie}{yyy,comp}
\icmlauthor{Ke Chen}{yyy,comp}
\icmlauthor{Lidan Shou}{yyy,comp}

\end{icmlauthorlist}

\icmlaffiliation{yyy}{The State Key Laboratory of Blockchain and Data Security, Zhejiang University}
\icmlaffiliation{comp}{Hangzhou High-Tech Zone (Binjiang) Institute of Blockchain and Data Security}
\icmlaffiliation{uw}{Paul G. Allen School of Computer Science \& Engineering, University of Washington}
\icmlcorrespondingauthor{Zhongle Xie}{xiezl@zju.edu.cn}
\icmlcorrespondingauthor{Lidan Shou}{should@zju.edu.cn}

\icmlkeywords{Machine Learning, ICML}

\vskip 0.3in
]

\printAffiliationsAndNotice{\icmlEqualContribution} 

\begin{abstract}

With the widespread adoption of Mixture-of-Experts (MoE) models, there is a growing demand for efficient inference on memory-constrained devices.
While offloading expert parameters to CPU memory and loading activated experts on demand has emerged as a potential solution, the large size of activated experts overburdens the limited PCIe bandwidth, hindering the effectiveness in latency-sensitive scenarios.
To mitigate this, we propose \model{}, an on-the-fly MoE inference system on memory-constrained GPUs.
\model{} is built on the insight that there exists substantial untapped redundancy within sparsely activated experts.
It employs various compression techniques on the expert's internal parameter matrices to reduce the data movement load, combined with low-cost sparse prediction, achieving perceptible inference acceleration in wall-clock time on resource-constrained devices.
Empirically, \model{} achieves a 9.3$\times$ compression of parameters per expert in Mixtral-8$\times$7B; enables deployment on a GPU with only 11GB VRAM, reducing the memory footprint by up to 8.5$\times$; and delivers a 48.7$\times$ inference speedup compared to DeepSpeed-MII on a single GeForce RTX 3090—all with only a 4.4\% $\sim$ 7.6\% average performance degradation.
\end{abstract}

\section{Introduction}
\label{introduction}


Mixture of Experts (MoE) models including DeepSeek-R1~\citep{deepseekai2025deepseekr1}, GPT-4~\citep{openai:2023gpt4}, 
Phi-4~\citep{abdin2024phi4}, Mixtral~\citep{albert2024mixtral}, etc., offer a paradigm shift in the large language model (LLM) architecture by introducing sparsely activated experts.
These sparse LLMs contextually activate only a subset of experts per token, significantly reducing inference costs while maintaining generative performance.
However, the abundance of idle, non-activated experts during MoE inference significantly hampers efficient GPU memory utilization, making it challenging to deploy MoE models on memory-constrained GPUs.
For instance, running inference for Mixtral-8$\times$7B, where two experts are activated, requires approximately 94GB of VRAM in \texttt{FP16} precision.
Of this, 30\% of the activated parameters (27.3GB) are utilized during decoding, while the remaining 66.8GB is occupied by non-activated experts, resulting in significant inefficiency~\citep{shin2024sparseinfer}.

\zl{
To address the problem, offloading techniques~\citep{edgemoe,art2023fasr,pre-gated,song2024promoe,xue2024moe-infinity,tang2024hobbit},
which unmount expert parameters to CPU memory and load them into GPU memory on demand for each input, offers a natural solution.
}
However, offloading shifts the decoding bottleneck from memory-bound to I/O-bound, as transferring billions of parameters through the low-bandwidth PCIe bus incurs substantial data transfer delays. 
\zl{
For comparison, the DRAM-to-VRAM bandwidth (32GB/s for PCIe 4.0) is orders of magnitude lower than the bandwidth between GPU memory and on-chip computation units (300GB/s).}
\zl{Consequently,} existing MoE inference systems with expert offloading, designed for edge-side continuous serving scenarios (i.e.,~single-batch latency-sensitive inference)~\cite{kong-etal-2024-swapmoe,art2023fasr,pre-gated,edgemoe,tang2024hobbit}, still fail to support \textit{on-the-fly} inference, where the loading process is \textit{perceptible} to the user because its overhead cannot be hidden by the model computation.
Ultra-low-bit quantization effectively reduces the size of transmitted parameters to mitigate the latency of activated expert loading~\citep{art2023fasr,edgemoe}, but at the cost of significantly degraded generation performance.
\zl{Thus, a pressing question emerges:}

\textit{How can we hide the I/O overhead of activated experts within model computation to enable on-the-fly MoE inference on the memory-constrained GPU while minimizing generation performance degradation?}

\zl{
In this paper, we present an on-the-\textbf{fl}y M\textbf{oE} inference system, coined \textbf{\model{}}, for consumer-grade devices.
\model{} reduces the I/O overhead of the experts, namely the transfer cost of the matrices for up, gate, and down projections, via a hybrid compression mechanism (\cref{sec:hybridcompression}).
Despite the utilization of the well-known inter-expert sparsity, the compression exploits the vast, untapped intra-expert sparsity in MoE models with a novel contextual sparsification scheme (\cref{ssec:sparse_gatedown}), balancing the transfer cost and the downstream performance.
In detail, the system first identifies low-magnitude, no-salient output activations of up projection and then removes the corresponding channel weights from the gate and down projections.
Meanwhile, we observe that the up projection matrix has limited sensitivity on performance against quantization, motivating us to enable the ultra-low-bit quantization in \model{} to reduce the transfer overhead further (\cref{ssec:quantize_up}).
}

\zl{
Although the hybrid compression reduces per-transfer cost for MoE models, the pipelining between transfer and computation is prevented due to the sequential execution of routing, quantized up projection computation, and DRAM expert fetching, inhibiting on-the-fly inference.
Therefore, we investigate the weights and the input during the computation and locate a high similarity between the shared hidden state input before routing and up projection of the MoE model.
Based on the finding, we devise two efficient yet effective sparsity predictors: (1) an inter-expert learning-based predictor to guide the routing of the activation expert of the next layer with the hidden state of the current layer; (2) an intra-expert reused-based predictor precomputing the context sparsity distribution with the hidden state of the current layer and the reused up projection.
The two predictors, with the help of prefetching, enable the pipelining of transfer and computation for on-the-fly inference. (\cref{sec:predictor})
}

\zl{
To integrate all the techniques above, we at last propose an efficient sparse kernel and compact asynchronous transfer from DRAM to VRAM to achieve system-wide efficiency (\cref{sec:sysoptim}).
The experimental study on various GPU specs and downstream tasks evidence the efficiency and efficacy of \model{} (\cref{sec:evaluation}).
Notably, for the popular Mixtral-8$\times$7B, \model{} achieves 9.3$\times$ parameter compression per expert, enables deployment on a GPU with just 11GB VRAM 
, and delivers a 2.6$\times$ inference speedup on an RTX 3090, with only 4.4\%$\sim$7.6\% average performance degradation.
}

\begin{figure*}[!t]
\begin{center}
\includegraphics[width=1.0\textwidth]{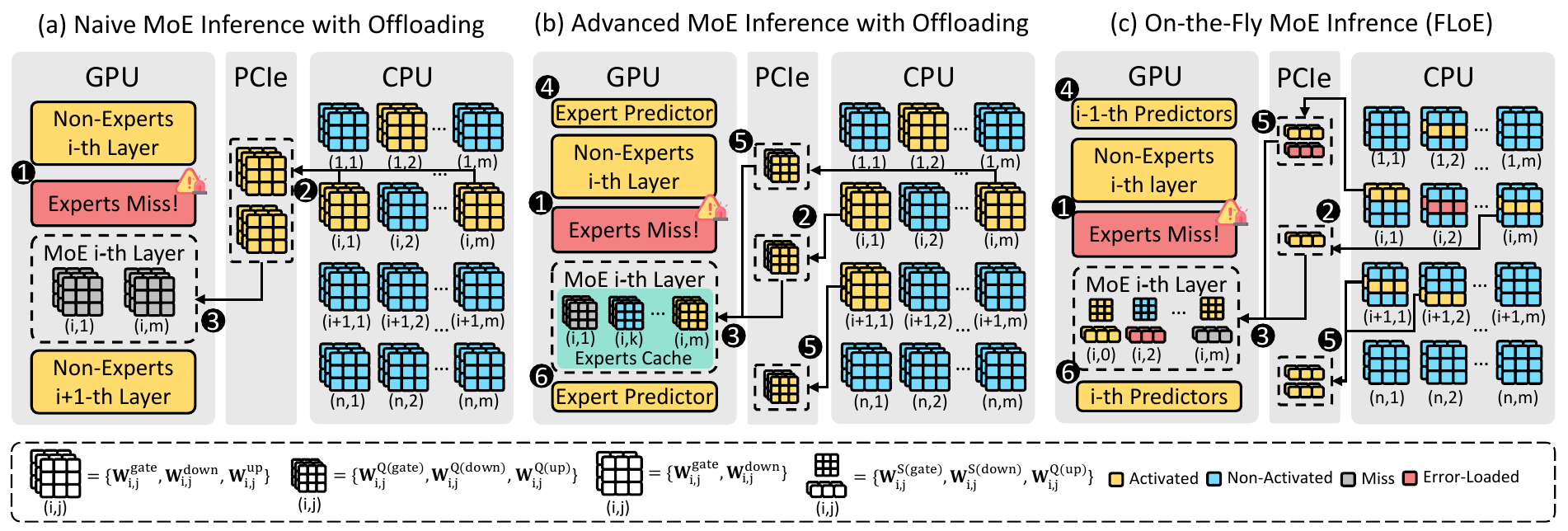}
\vskip -0.1in
\caption{\zj{Comparison of MoE inference offloading systems for memory-constrained GPUs: (a) Naive MoE Inference with Offloading; (b) Advanced MoE Inference with Offloading; (c) On-the-Fly MoE Inference (\model{}).}}
\label{fig:framework}
\end{center}
\vskip -0.3in
\end{figure*}

\section{Related Work}
\label{sec:related}
\paragraph{Experts Offloading.}
Efficient deployment of MoE models faces substantial challenges due to their parameter counts, particularly under resource constraints. Current inference frameworks like Llama.cpp~\citep{llama.cpp}, HuggingFace Accelerate~\citep{accelerate}, and DeepSpeed Inference~\citep{aminabadi2022deepspeed} employ experts offloading by selectively transferring VRAM-dominant expert weights to DRAM~\citep{sheng2023flexgen}. However, constrained PCIe bandwidth creates transfer bottlenecks during CPU-GPU expert transfers~\citep{kamahori2024fiddler}.  

To mitigate this, prefetching strategies predict and preload required experts through two paradigms: experience-based statistical methods using offline activation traces~\citep{edgemoe} (limited to top-1 expert activation strategy~\citep{fedus2022switch}), and intermediate result-driven approaches leveraging hidden states~\citep{art2023fasr,pre-gated,song2024promoe,tang2024hobbit} or prior expert indices~\citep{xue2024moe-infinity}. The former fails under multi-expert activation due to exponential path growth~\citep{dai2024deepseekmoe}, while the latter faces an accuracy-latency tradeoff: early-stage predictions from intermediate results~\citep{song2024promoe} diminish prefetching accuracy, necessitating costly expert reloads, whereas adjacent-layer predictions~\citep{pre-gated,art2023fasr} prevent computation-communication overlap.  

While ultra-low-bit expert quantization reduces transfer overhead at the cost of accuracy~\citep{art2023fasr,edgemoe}, CPU-based partial computation~\citep{kamahori2024fiddler,xue2024powerinfer2,tang2024hobbit} achieves limited acceleration due to insufficient throughput for high-dimensional matrix operations.

In contrast to our focus on on-the-fly inference in latency-sensitive scenarios, alternative offloading solutions, such as MoE-lightning~\citep{cao2024moe-lightning}, are primarily designed for high-throughput inference in offline scenarios.
\paragraph{Sparsity in LLMs.}
Maintaining model quality while minimizing parameter transfer necessitates synergistic sparsity and quantization. 
Weight pruning~\citep{sun2023wanda,frantar2023sparsegpt,ma2023llmpruner} zeroes subsets of LLM weights to reduce computational/memory overhead but faces performance degradation and hardware compatibility issues on consumer-grade devices.

Activation sparsity—conditional computation via zero-rich hidden states—naturally occurs in ReLU-based MLPs~\citep{liu2023deja,alizadeh2023llminflash,shin2024sparseinfer} but diminishes in modern architectures using non-ReLU MLPs (e.g., SwiGLU~\citep{shazeer2020glu}), limiting direct applicability. 
Recent research has thus concentrated on reintroducing activation sparsity within newer architectures~\citep{mirzadeh2023relu,zhang2024relu,song2024prosparse,song2024turbo}, but requires extensive pretraining (billions of tokens). Training-free activation sparsity~\citep{lee2024cats,liu2024teal}, achieved through activation magnitude pruning in SwiGLU-based LLMs, remains tailored for dense models with uniform parameter utilization across inputs.

\section{\model{}: On-the-Fly MoE Inference}
\label{sec:floe}
\subsection{MoE Inference with Offloading}

\cref{fig:framework}(a) illustrates naive MoE inference with offloading. Non-expert weights, frequently activated during inference, reside persistently in VRAM and are computed on the GPU. Expert weights, due to their sparse activation, are offloaded to DRAM. If certain experts' weights are missing from VRAM (\cref{fig:framework}(a) \ding{202}), the system transfers these weights over the PCIe bus (\cref{fig:framework}(a) \ding{203}), after which the GPU proceeds with subsequent computations (\cref{fig:framework}(a) \ding{204}).

\zl{
As mentioned in the introduction, expert offloading shifts the decoding bottleneck from memory-bound to I/O bound.
Specifically, the expert transferring from DRAM to VRAM incurs long latency. 
}
For example, \zl{the} expert in the Mixtral-8$\times$7B model has over 300MB of \texttt{FP16} parameters, taking nearly 15ms to transfer over a 16-channel PCIe 4.0 bus, \zl{whilst} a single expert’s computation on a GeForce RTX 3090 takes only about 5ms.

\cref{fig:framework}(b) shows advanced MoE offloading~\citep{art2023fasr,edgemoe,pre-gated,song2024promoe,xue2024moe-infinity,tang2024hobbit}(detailed related works discussion in~\cref{sec:related}).
\zl{
Despite the process in the naive solution, an extra expert predictor~(\cref{fig:framework}(b) \ding{205}) is implemented to prognosticate the expert visiting in the near future.
The prognosticated expert, as shown as \cref{fig:framework}(b) \ding{206}, is quantized and preloaded in a GPU-resident expert cache~(\cref{fig:framework}(b) \ding{204}), managed by a replacement policy.
Compared to the naive solution, the advanced MoE offloading can achieve better transfer efficiency due to the usage of the expert predictor and the cache.
}

\zj{
Next, we present \model{}, an inference system that delivers on-the-fly MoE model inference on consumer-grade GPUs. \model{} uses a hybrid compression scheme—integrating contextual sparsity and ultra-low-bit quantization (\cref{fig:framework}(c) \ding{207})—detailed in~\cref{sec:hybridcompression}. In \cref{sec:predictor}, \model{} introduces dual predictors (\cref{fig:framework}(c) \ding{205}) for inter- and intra-expert sparsity to accurately prefetch activated compressed weights (\cref{fig:framework}(c) \ding{206}) while minimizing DRAM usage. Finally,~\cref{sec:sysoptim} describes system co-optimizations that further enhance \model{}’s efficiency.}

\subsection{Expert Hybrid Compression}
\label{sec:hybridcompression}



\zj{As shown in~\cref{fig:framework}, in a SwiGLU-based MoE model, each expert $\mathcal{E}_{ij}$ consists of three matrices $\{\mathbf{W}^{\textnormal{gate}}_{ij}, \mathbf{W}^{\textnormal{down}}_{ij}, \mathbf{W}^{\textnormal{up}}_{ij}\}$.
\zl{
We denote the number of layers and the number of experts per layer as $m$ and $n$, respectively.
}
Although advanced MoE offloading
proposes compressing experts using ultra-low-bit quantization (e.g.,~\texttt{INT2}, \texttt{INT1}) to reduce transfer costs, this significantly degrades model performance.
}

\zj{
We argue that applying a uniform ultra-low-bit quantization strategy $\mathtt{Q}(\cdot)$ across all matrices 
(see~\cref{fig:framework}(b) 
$\{ 
\mathbf{W}^{\mathtt{Q}\textnormal{(gate)}}_{ij}, 
\mathbf{W}^{\mathtt{Q}\textnormal{(down)}}_{ij}, 
\mathbf{W}^{\mathtt{Q}\textnormal{(up)}}_{ij} \}$
) within an expert fails to strike an optimal balance between efficiency and performance.
Thus, \model{} introduces a unique twist with a hybrid strategy that tailors compression methods to the properties of the projection matrices.
Specifically, contextual activation sparsity $\mathtt{S}(\cdot)$ is applied to the gate \zl{projection $\mathbf{W}^{\textnormal{gate}}{ij}$ and down projection $\mathbf{W}^{\textnormal{down}}_{ij}$ to produce $\mathbf{W}^{\mathtt{S}\textnormal{(gate)}}_{ij}$ and $\mathbf{W}^{\mathtt{S}\textnormal{(down)}}_{ij}$. Meanwhile, ultra-low-bit quantization $\mathtt{Q}(\cdot)$ (\texttt{INT2}) is used for the up projection $\mathbf{W}^{\textnormal{up}}_{ij}$ to yield $\mathbf{W}^{\mathtt{Q}\textnormal{(up)}}_{ij}$.}
}


\subsubsection{Contextual Sparsification for Gate \& Down Projections}
\label{ssec:sparse_gatedown}



\begin{figure}[tbp]
\begin{center}
    \subfigure[$\mathbf{W}^{\textnormal{gate}}_{*}$ Frequency]{
        \includegraphics[width=0.326\linewidth]{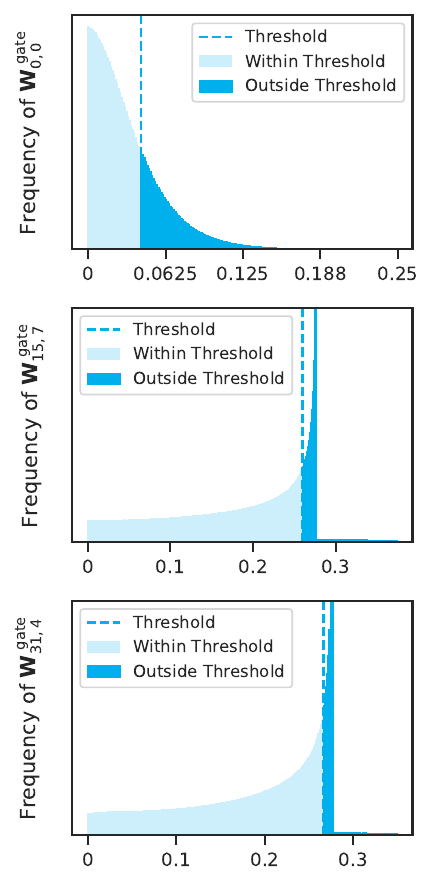}
        \label{subfig:quantization_sensitivity2}
    }%
    \subfigure[$\mathbf{W}^{\textnormal{down}}_{*}$ Frequency]{
        \includegraphics[width=0.326\linewidth]{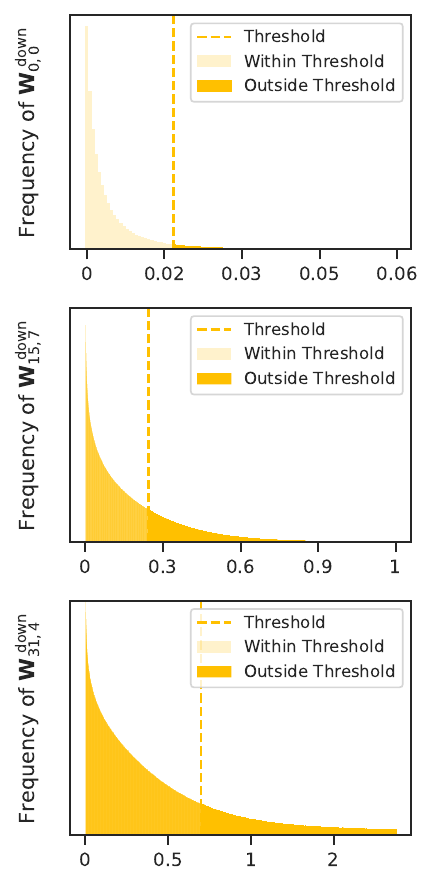}
        \label{subfig:sparsification_sensitivity2}
    }%
    \subfigure[$\mathbf{W}^{\textnormal{up}}_{*}$ Frequency]{
        \includegraphics[width=0.326\linewidth]{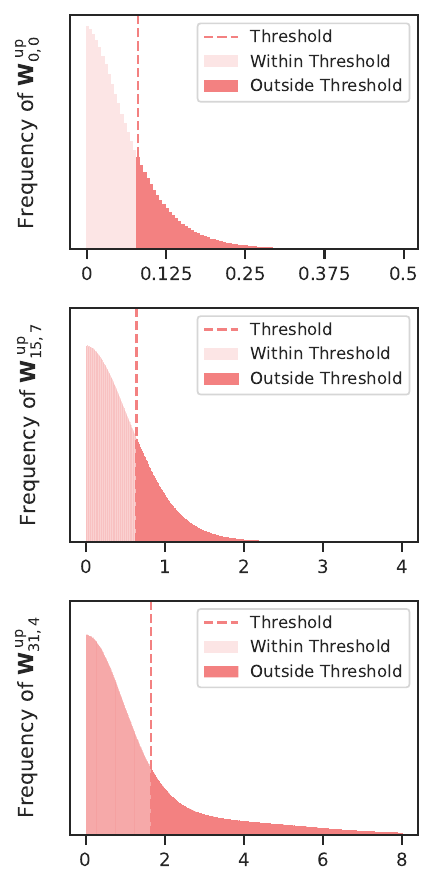}
        \label{subfig:up_activation}
    }%
    \caption{
\zj{Activation distributions of Mixtral-8$\times$7B’s three hidden states at experts $\mathcal{E}_{0,0}$ (shallow layer), $\mathcal{E}_{15,7}$ (middle layer), and $\mathcal{E}_{35,4}$ (deep layer).}}
    \label{fig:sparsity_ob}
\end{center}
\end{figure}

\zj{Contextual activation sparsity reduces model computations dependent on low-magnitude, non-salient contextual activations by pruning the corresponding channel weights, with minimal impact on model performance~\citep{liu2023deja,lee2024cats,liu2024teal}.
However, the MoE model already performs sparse activation inference through the router, selecting the experts to participate in the computation based on the given context.
}

\begin{question}\label{question_1}
\textit{
\zj{Does internal sparsity in experts of MoE models exist and persist consistently across layers?}}
\end{question}

\begin{observation}\label{observation_1}
\textit{
\zj{The experts within a sparsely activated MoE model maintain a 
high 
internal sparsity across layers.}}
\end{observation}

\zyx{
We conducted a preliminary study on the activation distribution within experts, analyzing output activations from the $\mathbf{W}^{\textnormal{gate}}$ and $\mathbf{W}^{\textnormal{up}}$ matrices, and input hidden states to the $\mathbf{W}^{\textnormal{down}}$ matrix of the Mixtral-8$\times$7B~\citep{albert2024mixtral} model on the C4 dataset~\citep{2019t5c4}, visualized in~\cref{fig:sparsity_ob}\footnote{Phi-3.5-MoE-instruct~\cite{abdin2024phi3technicalreporthighly} and DeepSeek-V2~\citep{deepseekv2} are validated in the~\cref{sec:sparAppen}.
}. Consistent with findings from CATS~\citep{lee2024cats} and TEAL~\citep{liu2024teal}, we observed that many activations are concentrated around zero. This concentration motivates the use of a magnitude-based activation sparse strategy, where activations close to zero are set to exactly zero, eliminating corresponding weight computations and transfers during inference.}


Given an input vector $\mathbf{x}$ and three projection weight matrices $\mathbf{W}^{\textnormal{gate}}, \mathbf{W}^{\textnormal{down}}, \mathbf{W}^{\textnormal{up}}$ in an expert $\mathcal{E}$, the corresponding activation output $\mathbf{a}_\mathcal{E}$ is computed as following forward pass:
\begin{equation}
\mathbf{a}_{\mathcal{E}}(\mathbf{x}) := \big(\texttt{SiLU}(\mathbf{x} \mathbf{W}^{\textnormal{gate}}) \odot (\mathbf{x} \mathbf{W}^{\textnormal{up}})\big) \mathbf{W}^{\textnormal{down}},
\label{eq:expert}
\end{equation}
\begin{equation}
\texttt{SiLU}(\mathbf{x}) := \mathbf{x} \cdot \sigma(\mathbf{x}) = \frac{\mathbf{x}}{1 + e^{-\mathbf{x}}},
\label{eq:silu}
\end{equation}
where $\odot$ denotes the Hardmard product and $\texttt{SiLU}(\cdot)$ is the activation function.
Therefore, magnitude-based sparsity can be determined from the outputs of the \texttt{SiLU} activation function, \(\mathbf{W}^{\textnormal{up}}\), and the inputs to \(\mathbf{W}^{\textnormal{down}}\). We define three activation functions:
\begin{equation}
\mathbf{a}_{\textnormal{gate}}(\mathbf{x}) = \texttt{SiLU}(\mathbf{x}\mathbf{W}^{\textnormal{gate}}), \quad 
\mathbf{a}_{\textnormal{up}}(\mathbf{x}) = \mathbf{x}\mathbf{W}^{\textnormal{up}},
\label{eq:activation}
\end{equation}
\begin{equation}
\mathbf{a}_{\textnormal{down}}(\mathbf{x}) = \mathbf{a}_{\textnormal{gate}}(\mathbf{x}) \odot \mathbf{a}_{\textnormal{up}}(\mathbf{x}),
\label{eq:up}
\end{equation}

and produce the following sparsity function:
\begin{equation}
\mathtt{S}_t(\mathbf{a}(\mathbf{x})) = 
\begin{cases} 
\mathbf{a}(\mathbf{x}), & \textnormal{if } |\mathbf{a(\mathbf{x})}| \geq t, \\
0, & \textnormal{if } |\mathbf{a(\mathbf{x})}| < t.
\end{cases}
\label{eq: s_t def}
\end{equation}
Here, \(\mathbf{a} \in \{\mathbf{a}_{\textnormal{gate}}, \mathbf{a}_{\textnormal{up}}, \mathbf{a}_{\textnormal{down}}\}\). The threshold \(t\) is derived from the sampled dataset based on the desired sparsity ratio:
\begin{equation}
t := \min\{t' : F(t') \geq k\},
\label{eq:silu_2}
\end{equation}
where \(F(\cdot)\) represents the empirical cumulative distribution function of absolute activation values for each expert, and \(k\) specifies the target sparsity ratio (e.g., 70\%). The distribution is empirically estimated offline using activations sampled from a general text corpus.

To evaluate the impact of magnitude-based activation pruning on model performance, we set thresholds for the outputs of the \texttt{SiLU} activation function, \(\mathbf{W}^{\textnormal{up}}\), and the inputs to \(\mathbf{W}^{\textnormal{down}}\) at various sparsity levels, then measured text perplexity on WikiText-2~\citep{wikitext2}. As shown in \cref{fig:sensitivity}(a), we find that pruning based on the \(\mathbf{W}^{\textnormal{down}}\) inputs is the least sensitive to sparsity: at 50\% sparsity, the perplexity increases by only about 0.5\%, and even at 90\% sparsity, perplexity remains relatively stable. In contrast, pruning the \(\mathbf{W}^{\textnormal{up}}\) outputs is slightly more sensitive, where 80\% sparsity roughly matches the 90\% sparsity level of the \(\mathbf{W}^{\textnormal{down}}\) inputs. Pruning the SiLU outputs is the most sensitive, pushing perplexity above 7 at 70\% sparsity. We provide further evaluations on downstream tasks in~\cref{exp:accuracy} and have a theoretical interpretation for this phenomenon (Refer to~\cref{sec: preliminary} for more details):
\newcommand{\informaltheorem}{
\begin{theorem}[informal]
    From the definition of \(\mathtt{S}_t\) in~\cref{eq: s_t def}, we define:
    \begin{align}
        \mathcal{L}_{\textnormal{down}} 
        &= \mathbb{E}\bigl\|\bigl(\mathbf{a}_{\textnormal{down}} - \mathtt{S}_t(\mathbf{a}_{\textnormal{down}})\bigr)\,\mathbf{W}^{\textnormal{down}}\bigr\|_2^2,\\
        \mathcal{L}_{\textnormal{up}} 
        &= \mathbb{E}\bigl\|\bigl(\mathbf{a}_{\textnormal{down}} - \mathbf{a}_{\textnormal{gate}} \odot \mathtt{S}_t(\mathbf{a}_{\textnormal{up}})\bigr)\,\mathbf{W}^{\textnormal{down}}\bigr\|_2^2,\\
        \mathcal{L}_{\textnormal{gate}} 
        &= \mathbb{E}\bigl\|\bigl(\mathbf{a}_{\textnormal{down}} - \mathtt{S}_t(\mathbf{a}_{\textnormal{gate}}) \odot \mathbf{a}_{\textnormal{up}}\bigr)\,\mathbf{W}^{\textnormal{down}}\bigr\|_2^2.
    \end{align}
    Then under assumptions consistent with experimental observations, we have
    \begin{equation}
        \mathcal{L}_{\textnormal{down}} \;\leq\; \mathcal{L}_{\textnormal{up}} \;<\; \mathcal{L}_{\textnormal{gate}}.
    \end{equation}
\end{theorem}
}
\informaltheorem
While the input pruning of down projection shows theoretical optimality for downstream tasks, its effectiveness is constrained by two factors: (1) Dependency on gate/up projection outputs limits computational savings to the final projection, and (2) Non-linear operations (\texttt{SiLU}, Hadamard product) hinder prediction for offloading. 
\zj{Empirical evaluations and theoretical analysis show that the output sparsity of up projection, compared to the \texttt{SiLU} activation function, yields superior generative performance at equivalent sparsity ratios. This motivates our design to replace the original expert forward pass computation in~\cref{eq:expert} as follows:}
\begin{equation}
\mathbf{a}^{\mathtt{S}}(x) := (\texttt{SiLU}(\mathbf{x} \mathbf{W}^{\textnormal{gate}}) \odot \mathtt{S}_t((\mathbf{x} \mathbf{W}^{\textnormal{up}}))) \mathbf{W}^{\textnormal{down}}
\label{eq:sparse_expert}
\end{equation}

\begin{figure}[tbp]
\begin{center}
    \subfigure[Sparsification Sensitivity]{
        \includegraphics[width=0.445\linewidth]{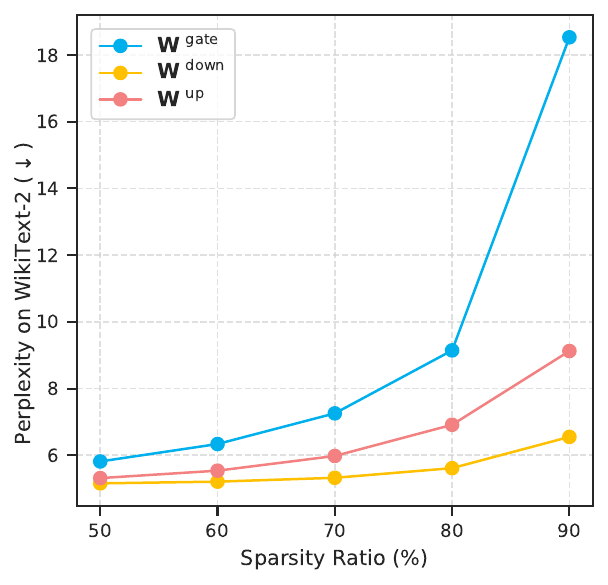}
        \label{subfig:sparsification_sensitivity1}
    }%
    \hfill
    \subfigure[Quantization Sensitivity]{
        \includegraphics[width=0.47\linewidth]{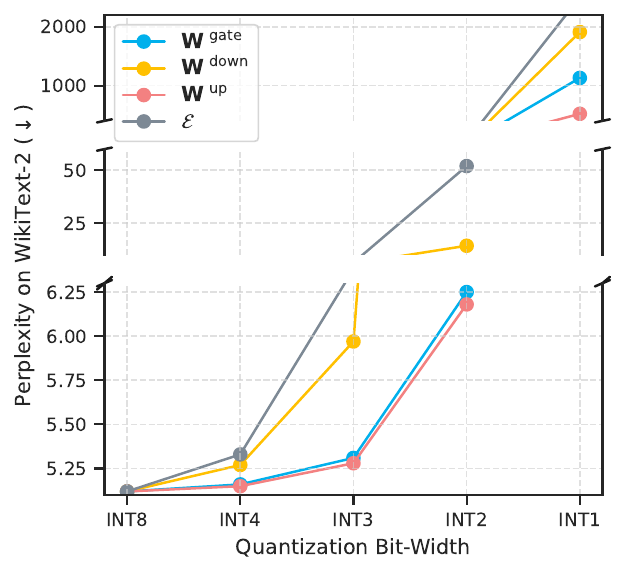}
        \label{subfig:quantization_sensitivity}
    }%
    \vspace{-10pt}
     \caption{Compression sensitivity of expert parameters: (a) Sparsification sensitivity; (b) Quantization sensitivity.}
    \label{fig:sensitivity}
\end{center}
    \vskip -0.2in
\end{figure}
\subsubsection{Ultra-low-bit Quantization for the Up Projection}
\label{ssec:quantize_up}


\zj{Thanks to contextual sparsity, only 10\% of the weights in the gate and down projections ($\{\mathbf{W}^{\textnormal{gate}},\mathbf{W}^{\textnormal{down}}\}$) are activated.
However, the full parameter set of the up projection $\mathbf{W}^{\textnormal{up}}$ is required for computation, as its output activations determine the sparsity threshold for truncating $\mathbf{W}^{\textnormal{gate}}$ and $\mathbf{W}^{\textnormal{down}}$. 
\zl{
As mentioned, prior work~\citep{art2023fasr} deploying MoE models on consumer-grade devices suffers from substantial performance loss due to the uniform ultra-low-bit quantization on the three projection matrices.
}
Building upon the fact that the contextual sparsity of gate and down projections has minimal impact on performance, thus alleviating the quantization burden on experts, the following question arises:
}

\begin{question}\label{question_2} \zj{\textit{Can we quantize only the full up projection, from $\mathbf{W}^{\textnormal{up}}$ to $\mathbf{W}^{\mathtt{Q}\textnormal{(up)}}$, and effectively reverse the inherent performance degradation caused by uniform quantization?}} \end{question}

\begin{observation}\label{observation_2}
\zj{\textit{
The up projection exhibits low sensitivity to ultra-low-bit quantization.}}
\end{observation}




We employ Half-Quadratic Quantization (HQQ)~\citep{badri2023hqq} with various bit-widths to quantize the three projection matrices within each expert of Mixtral~8$\times$7B and evaluate their quantization sensitivity using perplexity on WikiText-2~\citep{wikitext2}\footnote{Phi-3.5-MoE-instruct~\cite{abdin2024phi3technicalreporthighly}, DeepSeek-MoE-16B-Base~\citep{dai2024deepseekmoe} and Qwen1.5-MoE-A2.7B~\citep{qwen_moe} are validated in the~\cref{sec:quanAppen}.}. As shown in \cref{fig:sensitivity}(b), quantizing the projection matrices at \texttt{INT8} and \texttt{INT4} results in minimal performance impact, with perplexity changes under 3\%. At \texttt{INT3} and \texttt{INT2}, perplexity increases, with the down projection exhibiting the most significant change, followed by the gate projection, while the up projection remains the least sensitive. At \texttt{INT1}, up projection quantization yields only 46.01\% of the perplexity of gate projection quantization and 27.23\% of the perplexity for down projection quantization. Across all bit-widths, the up projection consistently shows the lowest perplexity. 

\textbf{Analysis:}  
Some works~\citep{geva-etal-2021-transformer,yu-ananiadou-2024-neuron} 
treat
the MLP layer
(i.e.,~the expert in MoE)
as a key-value memory model, where the up and gate projections serve as keys to selectively activate the values in the down projection, which stores knowledge related to the input. This theory aligns with our experimental results.
The down projection, storing knowledge as values, requires higher precision than the gate and up projection, as evidenced by its significant performance degradation across different quantization bit-widths. The gate projection, influenced by nonlinear activations, e.g.,~SwiGLU~\citep{shazeer2020glu}, demonstrates greater sensitivity at ultra-low bit-widths (\texttt{INT2}, \texttt{INT1}).

\textbf{Implementation:}  
The observation and analysis of quantization sensitivity above suggest that the up projection is the least sensitive to quantization, and therefore we choose to apply the \texttt{INT2} of HQQ method for its compression.



\subsection{Expert Sparsity Prediction}
\label{sec:predictor}

\zj{
In MoE models, each MoE layer uses a router to determine activated experts for each input hidden state $\mathbf{x}$, followed by a forward pass according to~\cref{eq:expert}. 
Although the hybrid compression reduces per-transfer cost for MoE models, the pipelining between transfer and computation is prevented due to the sequential execution of routing, quantized up projection computation, and DRAM expert fetching, inhibiting on-the-fly inference.
Recalling sparsity prediction in dense LLMs~\citep{liu2023deja,Lee2024InfiniGen}, the residual structure of the model leads to high similarity between hidden state inputs before consecutive MLP layers. 
This allows the hidden state of the $i$-th MLP layer to be fed into a trained predictor to forecast the sparsity distribution for the $(i+1)$-th layer. 
Inspired by this approach and considering the same inputs to the router and quantized up projection, we pose the following question:}

\begin{question}\label{question_3} \textit{
\zl{
Can the hidden states $\mathbf{x}$ of the existing layer be used in sparsity prediction in prefetching the activated compressed experts for the successive layer, replacing the router and up projection computations?
}
} 
\end{question}
 
\zj{Fortunately, in sparse MoE models, we empirically validate the core principle behind sparsity prediction:}
\begin{observation}\label{observation_3} \textit{\zj{The hidden states input to the router and up projection in consecutive MoE layers exhibit high similarity.}} \end{observation}
\zj{Specifically, we randomly sample 100 sequences of length 256 from the ShareGPT~\citep{sharegpt} and feed them into Mixtral-8$\times$7B. Then, we compute the average next layer similarity, defined as the cosine similarity between the hidden states before the $i$-th layer and $(i+1)$-th layer.
\cref{fig:cos-sim} shows that the next layer similarity consistently remains above 0.95, except for the first layer.}


\zl{
Building on the observation, we devise two efficient yet effective predictors for inter- and intra-expert sparsity.
They both consume the input hidden states of the existing layer and prefetch the activated compressed experts to be visited in the next layer, hence excluding the upcoming computation of router and up projection.
}

\begin{figure}[!t]
\begin{center}
\includegraphics[width=0.45\textwidth]{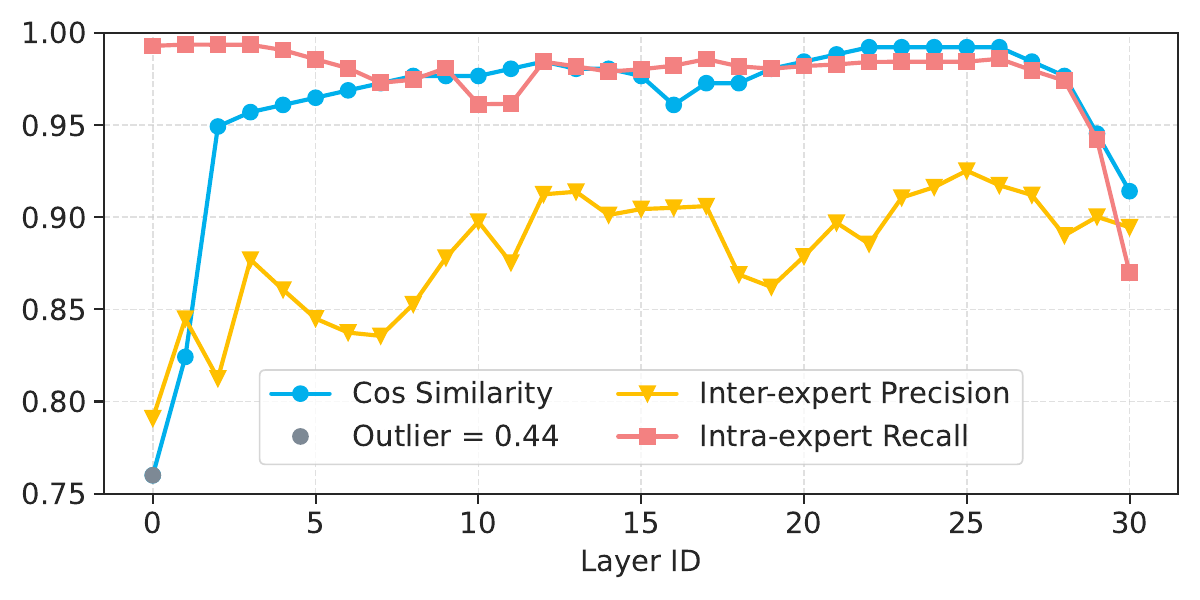}
\caption{Next layer cosine similarity (blue), intra-expert predictor precision (yellow), inter-expert predictor recall (red), and the outlier corresponding to the cosine similarity at the 0-th layer (gray).}
\label{fig:cos-sim}
\end{center}
\end{figure}

\subsubsection{Inter-expert sparsity predictor}
\label{sec:inter-expert}

\zj{For inter-expert sparsity, we introduce a learning-based predictor that proactively predicts the experts required for the $(i+1)$-th layer while computing the $i$-th layer.}

\zj{
The core idea of the learning-based predictor is to collect the input from the previous layer along with the historical trajectory of expert selections, capturing the underlying correlations between them.
Leveraging these correlations, the predictor makes informed decisions about future expert selections. 
We observe that the complexity of prediction diminishes as the layer depth increases. 
To adapt to this, we dynamically adjust the predictor's parameters in practice, scaling from a single-layer MLP with 32K parameters to a two-layer MLP with 2M parameters. 
The orange line in~\cref{fig:cos-sim} illustrates an average precision of 0.88, highlighting the inter-experts predictor's capability to maintain high accuracy while adapting to varying layer depths.
}

\subsubsection{Intra-expert sparsity predictor}

\zj{
For intra-expert sparsity, we introduce a parameter-free, reuse-based predictor. 
This predictor estimates the output activations of the up projection in the $(i+1)$-th layer by directly performing matrix multiplication between the hidden states before the $i$-th MoE layer and the reused up projection matrix of the $(i+1)$-th layer. 
Once the approximate output activations of the up projection are obtained, the contextual sparsity distribution can be computed in advance.
}

\zj{
Different from existing learning-based predictors~\citep{liu2023deja, shin2024sparseinfer, xue2024powerinfer2}, which, for example, impose an additional 2.19GB$\sim$9GB of memory footprint for models like Mixtral-8$\times$7B (detailed in~\cref{sec:cost_learningpredictor}), an unbearable burden for memory-constrained GPUs, our intra-expert sparsity predictor incurs little extra memory cost.
Furthermore, the red line in \cref{fig:cos-sim} shows an average recall of 0.95, demonstrating the predictor’s ability to maintain high accuracy across varying layer depths.
}


\begin{algorithm}[t]
   \caption{Efficient Sparse Kernel}
   \label{alg:Efficient_Sparse_Kernel}
\begin{algorithmic}[1]
    \STATE \textbf{Input:} hidden states $\mathbf{x}$, threshold $t_{ij}$, $\mathcal{E}_{ij} = \{ \mathbf{W}^{\textnormal{gate}}_{ij}, \mathbf{W}^{\textnormal{down},\top}_{ij}, \mathbf{W}^{\textnormal{up}}_{ij}\}$
    \STATE $\mathbf{v} \gets \mathbf{x} \mathbf{W}^{\textnormal{up}}_{ij}$ 
    \STATE $\mathbf{mask} \gets \left( \left| \mathbf{v} \right| > t_{ij} \right)$ 
    \STATE $\mathbf{x}' \gets \textnormal{SiLU}\left(\mathbf{x} \mathbf{W}^{\textnormal{gate}}_{ij}[\mathbf{mask}]  \right) \odot \mathbf{v}[\mathbf{mask}]$
    \STATE $\mathbf{y} \gets ( \mathbf{W}^{\textnormal{down},\top}_{ij}[\mathbf{mask}] \mathbf{x'})^\top$ 
    \STATE \textbf{Return:} $\mathbf{y}$
\end{algorithmic}
\end{algorithm}

\subsection{System Co-optimization}
\label{sec:sysoptim}
\subsubsection{Efficient Sparse Kernel}

To translate the reduction in computational complexity introduced by sparsity into clock time acceleration, we developed a specialized sparse GEMV kernel using the Triton~\citep{triton}-based kernel introduced by CATS~\citep{lee2024cats}. We achieve maximal data read efficiency by transposing $\mathbf{W}^{\textnormal{down}}_{ij}$ and utilizing column-major storage. By selectively loading the columns of the weight matrices $\mathbf{W}^{\textnormal{gate}}_{ij}$ and $ \mathbf{W}^{\textnormal{down},\top}_{ij} $ based on a threshold, we reduce the number of memory accesses, thereby accelerating clock time.

As shown in~\cref{alg:Efficient_Sparse_Kernel},  this kernel accepts the input hidden state $ \mathbf{x} $, sparse threshold $ t_{ij} $, and expert weights $  \mathcal{E}_{ij} = \left\{ \mathbf{W}^{\textnormal{gate}}_{ij}, \mathbf{W}^{\textnormal{down},\top}_{ij}, \mathbf{W}^{\textnormal{up}}_{ij} \right\} $. First, a mask vector is generated based on the absolute values of the hidden vectors output by $\mathbf{x}\mathbf{W}^{\textnormal{up}}$ and the magnitude of the threshold. The \texttt{SiLU} activation and element-wise multiplication are fused into each block computed by $ \mathbf{W}^{\textnormal{gate}}_{ij}[\mathbf{mask}] \mathbf{x}$, which conserves memory operations required for multiple storage and loading of $ \mathbf{x'} $ and reduces kernel launch time. Subsequently, the resulting $ \mathbf{x'} $ is multiplied by the transposed $ \mathbf{W}^{\textnormal{down},\top}_{ij} $ to produce the output of the sparse MLP.
~\cref{exp:latency} shows our sparse GEMV kernel effectively reduces expert computation time as sparsity increases.

\subsubsection{Compact Asynchronous Transfer}
Due to the sparse activation of weights, the expert transfer process occurs across multiple non-contiguous memory blocks in DRAM and VRAM, making it difficult to fully utilize the PCIe bus bandwidth. The Pytorch~\citep{pytorch} naive implementation can only achieve a fraction of the PCIe bandwidth, significantly affecting inference latency. Therefore, we \zyx{compacted} the arrangement of the gate and down projection matrices to transfer data in larger chunks and further enhance data throughput using SIMD and multithreaded asynchronous transfer techniques. 
Next,
we detail the optimization of the data transfer strategy.
\paragraph{Compact weights Layout}
In an expert, the activation of the $i$-th intermediate neuron corresponds to the usage of the $i$-th column from the gate \zyx{and up} projection matrix, along with the $i$-th row from the down projection matrix. By co-locating the corresponding columns of the gate projection and rows of the down projection in DRAM, we can \zyx{compact} the data into larger contiguous chunks for efficient transfer. Assuming each element of the $d_{\textnormal{hidden}} \times d_{\textnormal{intermediate}}$ weight matrices is stored in \texttt{num\_bytes}, this \zyx{layout} strategy increases the chunk size from $d_{\textnormal{hidden}} \times \texttt{num\_bytes}$ to $2 d_{\textnormal{hidden}} \times \texttt{num\_bytes}$, as illustrated in~\cref{fig:io}.
\begin{figure}[!t]
\begin{center}
\includegraphics[width=0.4\textwidth]{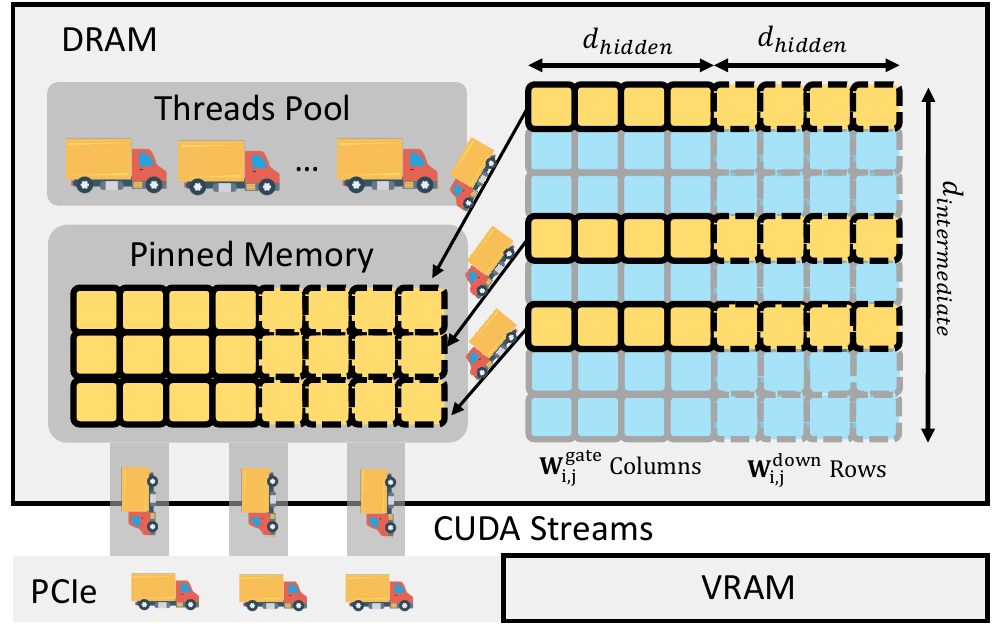}
\caption{\zyx{Process} of \model{}'s compact asynchronous transfer: \zyx{compacting} weights layout in DRAM for reduced access latency and multi-threaded packaging of activated experts to enable asynchronous data transfer.}
\label{fig:io}
\end{center}
\end{figure}
\paragraph{SIMD Asynchronous Transfer}
To fully leverage modern CPU capabilities, we use the AVX-512 instruction set and multithreaded asynchronous transfer. As shown in the~\cref{fig:io}, we allocate a pinned memory in the CPU (transferring data from pinned memory significantly improves transfer speed) to send weights to VRAM. 
We use multithreading in combination with SIMD instructions to bundle several weights groups for transfer into pinned memory, and asynchronously send transfer requests across multiple streams, minimizing idle time on the PCIe bus.
\section{Evaluation}
\label{sec:evaluation}
In this section, we aim to demonstrate that \model{} can speed up MoE decoding on limited GPU memory while preserving high accuracy. 
We first present our end-to-end system results showing wall-clock performance, followed by \model{}’s accuracy in downstream tasks\footnote{\zl{
We employ Mixtral-8$\times$7B as the MoE model for all test cases.
}}. Specifically,
\begin{itemize}[leftmargin=5pt]
    \item \lz{In~\cref{exp:latency}, we demonstrate that \model{} enables 48.7× end-to-end acceleration compared to DeepSpeed-MII, with sparse kernel contributing up to 2x speedup and compact asynchronous transfer achieving 12.6x faster performance compared to the naive method.}
    \item In~\cref{exp:accuracy}, we show that \model{} achieves a performance gain of \lz{9.8\%} over other methods at high sparsity.
\end{itemize}

\subsection{Efficiency Evaluation}\label{exp:latency}
\begin{figure*}[!t]
\begin{center}
\includegraphics[width=0.95\textwidth]{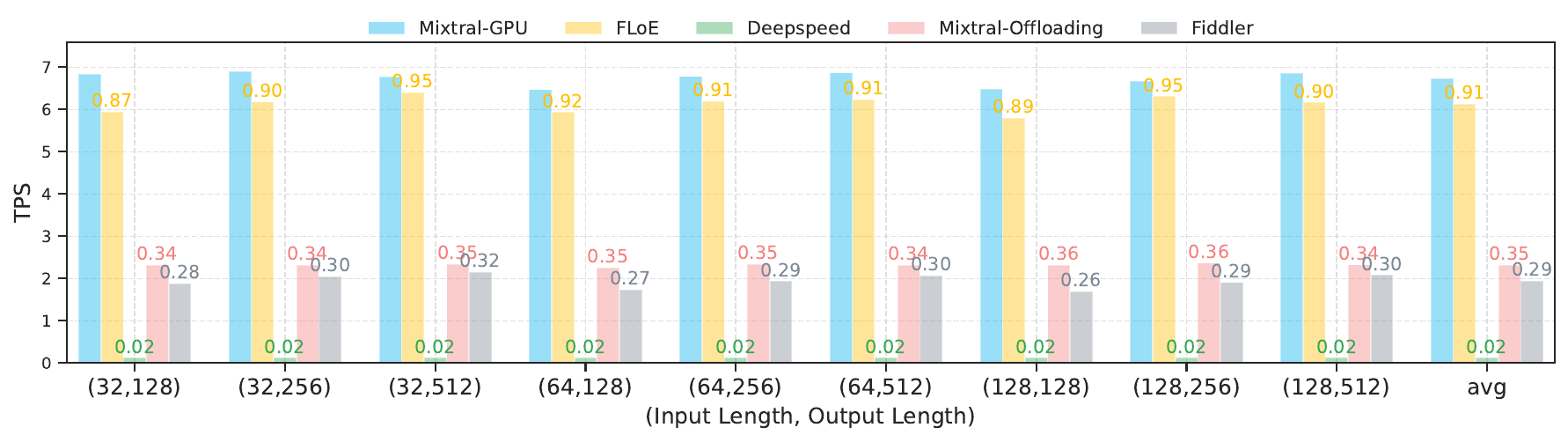}
\vspace{-10pt} 
\caption{Bars quantify generation speed of compared methods under 12GB VRAM constraints, with numerical labels indicating relative speedup ratios against the Mixtral-GPU baseline. Note that DeepSpeed uses \texttt{FP16} offloading.}
\label{fig:mainresult}
\end{center}
\vskip -0.2in
\end{figure*}

\begin{table*}[t]
\caption{Single-Expert Execution Latency with Sparse GEMV Kernel (ms).}
\label{tab:latency}
\begin{center}
\begin{small}
\begin{sc}
\begin{tabular}{@{}lcccccc@{}}
\toprule
GPU Model & \texttt{0\%} & \texttt{50\%} & \texttt{60\%} & \texttt{70\%} & \texttt{80\%} & \texttt{90\%} \\
\midrule
H100      & 0.169 & 0.134 (1.26$\times\uparrow$) & 0.123 (1.37$\times\uparrow$) & 0.114 (1.48$\times\uparrow$) & 0.106 (1.59$\times\uparrow$) & 0.103 (1.64$\times\uparrow$) \\
A100      & 0.253 & 0.195 (1.30$\times\uparrow$) & 0.188 (1.35$\times\uparrow$) & 0.176 (1.44$\times\uparrow$) & 0.166 (1.52$\times\uparrow$) & 0.155 (1.63$\times\uparrow$) \\
A6000     & 0.524 & 0.365 (1.44$\times\uparrow$) & 0.337 (1.56$\times\uparrow$) & 0.305 (1.72$\times\uparrow$) & 0.277 (1.89$\times\uparrow$) & 0.263 (1.99$\times\uparrow$) \\
RTX-3090  & 0.542 & 0.379 (1.43$\times\uparrow$) & 0.354 (1.53$\times\uparrow$) & 0.316 (1.72$\times\uparrow$) & 0.302 (1.80$\times\uparrow$) & 0.283 (1.92$\times\uparrow$) \\
\bottomrule
\end{tabular}
\end{sc}
\end{small}
\end{center}
\end{table*}

\zyx{
We analyze decode efficiency via end-to-end generation tests across various input/output lengths and VRAM usage, assess single-expert latency speedup for sparse GEMV, and evaluate transfer efficiency by simulating single-expert transfer with varying chunk sizes.
}

\paragraph{Setup}
We use GeForce RTX 3090 with 24G VRAM to evaluate end-to-end latency on ShareGPT~\citep{sharegpt} prompts.
The system is also equipped with a 64-core CPU at 2.3GHz and 256G DRAM interconnected via PCIe 4.0.
For the single-expert latency test, we use C4 dataset \citep{2019t5c4} and employ four types of GPUs, including H100, A100, A6000, and GeForce RTX 3090.


\vskip -0.2in
\paragraph{Baseline}
\label{sec:efficient_baselines}
We employ four SOTA baselines in the evaluation:
\textbf{DeepSpeed-MII} \citep{microsoft_deepspeed_mii}: 
\zl{An inference system utilizes ZeRO-Infinity \citep{zero} to deal with expert offloading.}
\textbf{Mixtral-Offloading} \citep{mixtral-offloading}: 
An MoE framework integrating expert prediction, caching mechanisms, and quantization.
\textbf{Fiddler} \citep{kamahori2024fiddler}: A CPU-GPU co-execution system minimizing data transfer overhead through computational offloading.
\textbf{Mixtral-GPU}: A model with HQQ \texttt{INT2} quantized enabling complete GPU residency, serving as the latency lower-bound reference for on-the-fly scenario requirements.

\vspace{-15pt}
\paragraph{Analysis}
We evaluate \model{}'s end-to-end efficiency with varying input/output lengths, and the results averaging over 5 runs are depicted in \cref{fig:mainresult}.
In the figure, we measure the inference speed for single-batch generation.
We select the average output tokens per second (TPS) as the measurement.
As seen, \model{} achieves 91\% of Mixtral-GPU’s speed (95\% at most), delivering 48.7$\times$, 2.60$\times$and 3.14$\times$ speedups over DeepSpeed-MII, Mixtral-Offloading and Fildder, respectively.
It should be noted that with longer outputs for fixed inputs, TPS improves as layer-wise expert replacement overhead is amortized over longer sequences.

\zl{
\cref{fig:vram} compares the generation throughput under input/output length of 64/256 and VRAM usage ranging from 12GB to 24GB.
}
With additional VRAM, we cache more MoE layers to reduce expert misprediction reload overhead.
Meanwhile, our sparse GEMV kernel applied to expert activations further boosts generation speed.
Across different VRAM capacities, our method remains close to Mixtral-GPU’s performance and slightly surpasses it at 24GB.
When DRAM usage reaches 21GB, Mixtral-Offloading essentially mirrors the Mixtral-GPU setup but is marginally slower, as it still relies on \zl{\texttt{INT3}} quantization for certain experts.

\zyx{\cref{tab:latency} compares the sparse kernel’s speedup across sparsity levels and \zl{GPUs} for a single expert, including dense up projection GEMV, fused SiLU activation, sparse gate projection GEMV, and sparse down-projection GEMV. 
Using 500 tokens from C4 dataset, we ran 80 warm-up iterations and 200 timed trials to measure execution latency.
Our kernel consistently outperforms the dense baseline (sparsity = 0). 
\zl{
At 50\% and 70\% sparsity, it achieves over 1.26x and 1.44x speedup, respectively.}
At 90\% sparsity, only A6000 and RTX 3090 \zl{obtain} nearly 2× speedup, while H100 and A100 are limited by kernel launch overhead and other non-computational factors due to their higher computational throughput.
\zl{The results evidence our sparse GEMV kernel is advantageous on consumer-grade devices.}
}

\zyx{For transfer efficiency,} we randomly selected 20\% of expert weights (20\% of columns in the gate projection matrix and corresponding columns in the transposed down projection) and transferred them from DRAM to VRAM using varying chunk sizes (number of weight columns per thread). The average over 20 trials is shown in~\cref{fig:iotest}. Our compact asynchronous transfer achieve up to 88\% of peak bandwidth, 12.6× faster than PyTorch~\citep{pytorch} native implementation. Compact weights layout improves efficiency across all chunk sizes. Transfer latency first increases and then decreases as chunk size grows—small chunks are dominated by API calls and CUDA launch overhead, while large chunks suffer from excessive DRAM packing time, limiting transfer overlap. The optimal chunk size in our setup is 50.
\begin{figure}[!t]
\begin{center}
\includegraphics[width=0.45\textwidth]{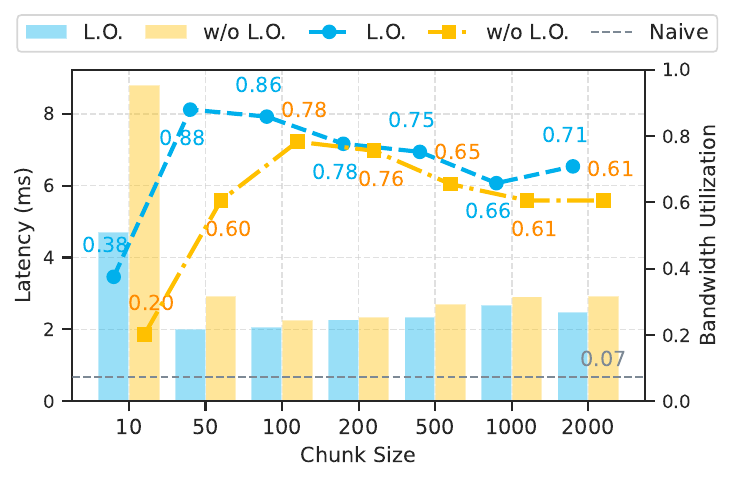}
\vspace{-15pt} 
\caption{Comparison transfer latency and bandwidth utilization: bars show DRAM-to-VRAM 
transfer delays per expert, while lines depict utilization relative to PCIe 4.0's actual peak bandwidth. Gray dashed lines are PyTorch's native implementation.}
\label{fig:iotest}
\end{center}
\vskip -0.2in
\end{figure}


\begin{figure}[!t]
\begin{center}
\includegraphics[width=0.48\textwidth]{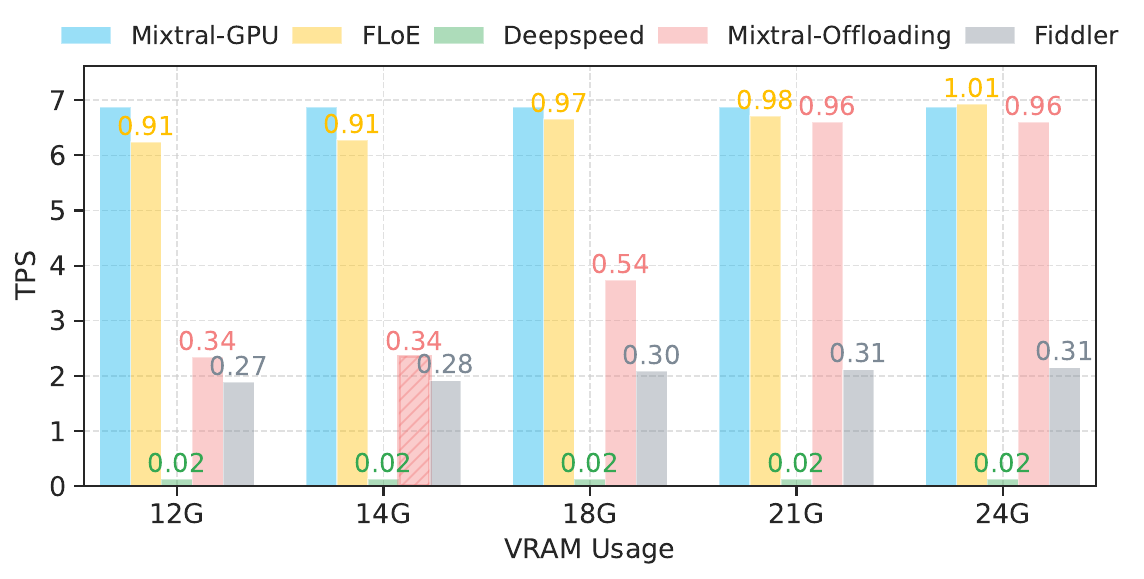}
\vskip -0.2in
\caption{Illustration of the generation speed of different methods under varying DRAM usage, with numbers indicating the speed relative to Mixtral-GPU.Since Mistral-Offloading caches by layer, there is no configuration for 14GB DRAM usage, so we use the 12GB result instead.}
\label{fig:vram}
\end{center}
\vskip -0.2in
\end{figure}

\begin{figure}[tbp]
\begin{center}
    
    \subfigure[Sparsification Sensitivity]{
        \includegraphics[width=0.46\linewidth]{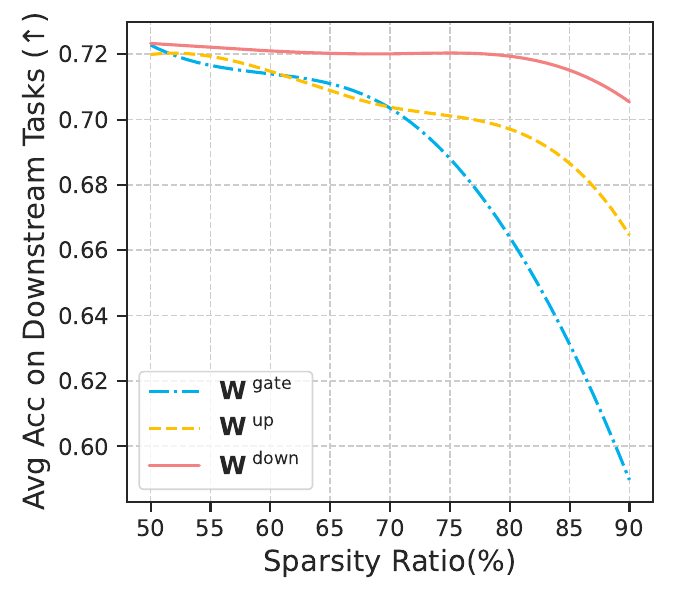}
        \label{subfig:quant-performance}
    }%
    \hfill
    \subfigure[Quantization Compatibility]{
        \includegraphics[width=0.46\linewidth]{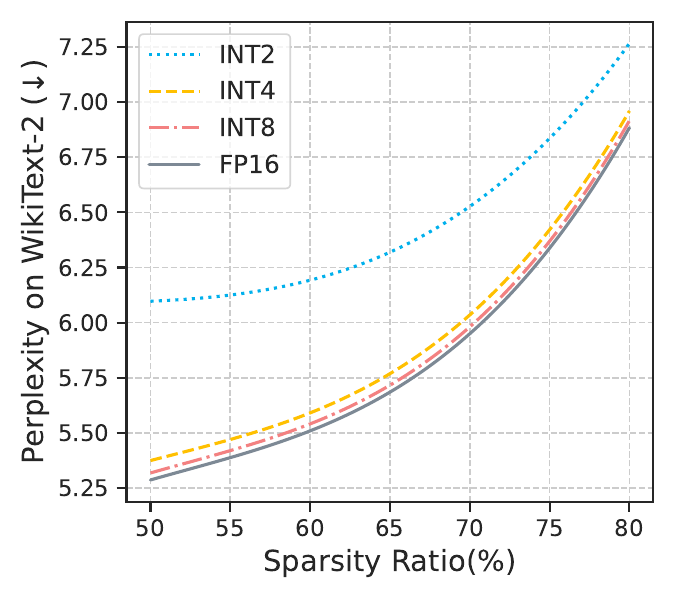}
        \label{subfig:quant-sparsity}
    }%
    \vskip -0.1in
     \caption{\zj{Impact of scaling up sparsity ratio on performance: (a) Task accuracy across different sparsity strategies, and (b) Text perplexity of \model{} combined with various quantization bit-widths.}}
    \label{fig:evaluation_end}
\end{center}
    \vskip -0.2in
\end{figure}


\subsection{Efficacy Evaluation}\label{exp:accuracy}
We analyze model efficacy via downstream tasks and validate the compatibility of the quantization with \model{}.

\paragraph{Setup}
For the downstream task performance, we use seven downstream tasks using the EleutherAI LM Harness~\citep{eval-harness}, \lz{including zero-shot ARC easy and challenge, zero-shot BoolQ, zero-shot SciQ, zero-shot OpenBookQA, zero-shot Winogrande and 5-shot MMLU~\citep{clark2018arcc,clark2019boolq,SciQ,sakaguchi2019winogrande,hendrycks2021mmlu}.}
These tasks are originally chosen to measure the abilities of the models across various domains, such as reading comprehension and reasoning.



\paragraph{Baseline}

\label{sec:efficacy_baselines}
\lz{
We employ three sparsity or quantization baselines in the evaluation:
\textbf{CATS}~\citep{lee2024cats}: \zl{A SOTA} activation sparsification method, which applies magnitude pruning to FFN activations. \textbf{CHESS}~\citep{he2024chess}: A general activation sparsification approach via channel-wise thresholding and selective sparsification. 
\textbf{HQQ quantization}~\citep{badri2023hqq}: A fast and accurate model quantizer that skips the need for calibration data. 
}
\paragraph{Analysis}
\lz{As shown in~\cref{fig:downstreamtask}, \model{}-$\mathbf{W}^{\textnormal{up}}$ achieves a 2.8\% accuracy improvement at 80\% sparsity and a significant performance gain of 9.8\% over the SOTA methods at 90\% sparsity. 
The reason lies in the fact that activations for $\mathbf{W}^{\textnormal{up}}$ demonstrate better performance compared to those for $\mathbf{W}^{\textnormal{gate}}$, as evidenced by the trend in~\cref{fig:evaluation_end}(a). In addition, \model{} combines both quantization and sparsity, trading off a small amount of accuracy for improved deployment speed. Despite this trade-off, its performance remains higher than that of HQQ \texttt{INT3} and CHESS.
On the MMLU task, we observe a noticeable performance drop from 
\model{}-$\mathbf{W}^{\textnormal{up}}$ to \model{}. 
When GPU memory is sufficient, increasing the bit-width of the $\mathbf{W}^{\textnormal{up}}$ matrix can mitigate this issue.
}



\begin{figure}[!t]
\begin{center}
\includegraphics[width=0.46\textwidth]{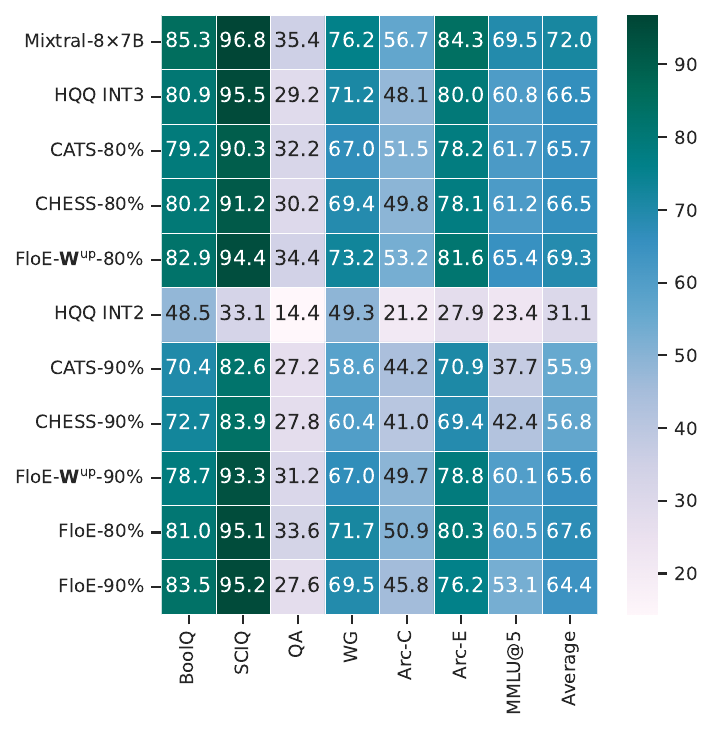}
\vskip -0.2in
\caption{Downstream task performance. \model{}-$\mathbf{W}^{\textnormal{up}}$ refers to our contextural sparsificaition according up projection.}
\label{fig:downstreamtask}
\end{center}
\vskip -0.3in
\end{figure}

We also demonstrate the compatibility with quantization techniques by evaluating different HQQ quantizations and plotting the perplexity variations of Mixtral-8$\times$7B on WikiText-2 in~\cref{fig:evaluation_end}(b). The perplexity increases exhibit similar trends across different bit widths, indicating that the errors introduced by activation sparsity and weight quantization are largely independent and additive. 

\section{Conclusion}

\lz{
We introduce \model{}, an on-the-fly inference system for MoE models on memory-constrained GPUs, which optimizes GPU memory utilization through an expert hybrid compression scheme and 
effective sparsity predictors, achieving a remarkable 48.7× inference speedup on a single GeForce RTX 3090 compared to DeepSpeed-MII.
We hope our work inspires further research on MoE inference with offloading from a sparsity perspective and believe \model{} will serve as a valuable tool for the community, enabling on-the-fly inference of sparse MoE models on consumer-grade hardware.}





\section*{Acknowledgements}

This work was supported by the Pioneer R\&D Program of Zhejiang (No. 2024C01021), the National Regional Innovation and Development Joint Fund (No. U24A20254), and the Zhejiang Province 'Leading Talent of Technological Innovation Program' (No. 2023R5214).

\section*{Impact Statement}
This paper presents work whose goal is to advance the field of Machine Learning. There are many potential societal consequences of our work, none which we feel must be specifically highlighted here.



\nocite{langley00}

\bibliography{ICML2025}
\bibliographystyle{ICML2025}

\newpage
\appendix
\onecolumn

\clearpage
\section{Theoretical Analysis}

\subsection{Preliminary}
\label{sec: preliminary}
Given a vector $x \in \mathbb{R}^m$, we use $\|x\|_2 := \sqrt{\sum_{i=1}^m x_i^2}$ to denote the two-norm of $x$.
We use $[m]$ to represent the set $\{1,2,...,m\}$.
We use $\odot$ to denote the element-wise multiplication of two vectors or matrice.

We denote $\mathcal{N}(\mu, \sigma^2)$ to be the Gaussian distribution with mean $\mu$ and variance $\sigma^2$.
We let $ \phi(x) = \frac{1}{\sqrt{2\pi}} e^{-x^2 / 2} $ and $\Phi(x) = \int_{-\infty}^{x} \phi(y)dy$ to be the probability density function (PDF) and cumulative distribution function (CDF)  of the standard normal distribution, respectively. Given function $f$, we use $f^{-1}$ to define its reverse function.

\subsection{Main Theorem}
In the main paper, we state the informal theorem as below:

\informaltheorem


Here we restate the  theorem in a formal format.

\begin{theorem}\label{thm: formal}
    Let \(\mathbf{a}_{\textnormal{gate}} \in \mathbb{R}^m\) and \(\mathbf{a}_{\textnormal{up}} \in \mathbb{R}^m\) be the activations after the \texttt{SiLU}
    function 
    and the up projection, respectively, and define \(\mathbf{a}_{\textnormal{down}} = \mathbf{a}_{\textnormal{gate}} \odot \mathbf{a}_{\textnormal{up}}\). Let \(\mathbf{W}^{\textnormal{down}} \in \mathbb{R}^{m \times n}\) be the weight matrix for the down projection. 
    From the definition of \(\mathtt{S}_t\) in~\cref{eq: s_t def}, we define:
    \begin{align}
        \mathcal{L}_{\textnormal{down}} 
        &= \mathbb{E}\bigl\|\bigl(\mathbf{a}_{\textnormal{down}} - \mathtt{S}_t(\mathbf{a}_{\textnormal{down}})\bigr)\,\mathbf{W}^{\textnormal{down}}\bigr\|_2^2,\\
        \mathcal{L}_{\textnormal{up}} 
        &= \mathbb{E}\bigl\|\bigl(\mathbf{a}_{\textnormal{down}} - \mathbf{a}_{\textnormal{gate}} \odot \mathtt{S}_t(\mathbf{a}_{\textnormal{up}})\bigr)\,\mathbf{W}^{\textnormal{down}}\bigr\|_2^2,\\
        \mathcal{L}_{\textnormal{gate}} 
        &= \mathbb{E}\bigl\|\bigl(\mathbf{a}_{\textnormal{down}} - \mathtt{S}_t(\mathbf{a}_{\textnormal{gate}}) \odot \mathbf{a}_{\textnormal{up}}\bigr)\,\mathbf{W}^{\textnormal{down}}\bigr\|_2^2.
    \end{align}
    We assump that all the $W_{ij}$ in $\mathbf{W}^{\textnormal{down}}$ are \textit{i.i.d.} and satisfies $W_{ij} \sim \mathcal{N}(0, \sigma_W^2)(i \in [m], j \in [n])$.Similarly, for all $i \in [m]$, we assume $a_{\text{gate},i}$ are $i.i.d.$ and satisfies $a_{\text{gate},i} \sim \mathcal{N}(0, \sigma_{\text{gate}}^2)$ . And for all $i \in [m]$, we let $a_{\text{up},i}$ are $i.i.d.$ and $a_{\text{gate},i} = x_{\text{gate},i} - c$ for some constant $c > 0$, where $x_{\text{gate},i}$ satisfies exponential distribution with parameter $\lambda$. 
    We also assume $a_{\text{up}}$ and $a_{\text{gate}}$ are independent. 
    Then if we keep the threshold of sparsity such that $(1-\eta) \times 100\%$ elements of the activations are set to zero in \(\mathtt{S}_t\), we can explictly write out $\mathcal{L}_{\textnormal{up}}$ and $\mathcal{L}_{\textnormal{gate}}$ as follows:
    \begin{equation}
        \mathcal{L}_{\textnormal{up}} = n m \sigma_W^2 \cdot  \sigma_{\text{up}}^2 (\frac{2}{\lambda^2} - \frac{2c}{\lambda} + c^2)\cdot \Big( 1 - \eta - 2 z_{\eta} \phi(z_{\eta}) \Big). 
    \end{equation}
    \begin{equation}
        \mathcal{L}_{\textnormal{up}} = n m \sigma_W^2 \cdot  \sigma_{\text{up}}^2 \cdot\Big[e^{\lambda (q_\eta-c)} \big(  
        \frac{2}{\lambda^2} 
        - 2 \frac{q_\eta}{\lambda}
        + q_\eta^2 
        \big) - e^{-\lambda (c + q_\eta)} \big( 
         \frac{2}{\lambda^2} 
         + 2 \frac{q_\eta}{\lambda}
        + q_\eta^2
        \big)\Big], 
    \end{equation}
    where $z_{\eta} = \Phi^{-1}\left(1 - \frac{\eta}{2}\right)$, $ \phi(x) = \frac{1}{\sqrt{2\pi}} e^{-x^2 / 2}$ and $q_\eta = \frac{1}{\lambda c}\sinh^{-1}(\frac{1-\eta}{2}e^{\lambda c})$.
    Furthermore, if $\lambda c \geq 2$, and $\eta \in [e^{-4}, 1/2]$, we can obtain
    \begin{equation}
        \mathcal{L}_{\textnormal{down}} \;\leq\; \mathcal{L}_{\textnormal{up}} \;<\; \mathcal{L}_{\textnormal{gate}}.
    \end{equation}
\end{theorem}

\begin{figure}
    \centering
    \includegraphics[width=0.5\linewidth]{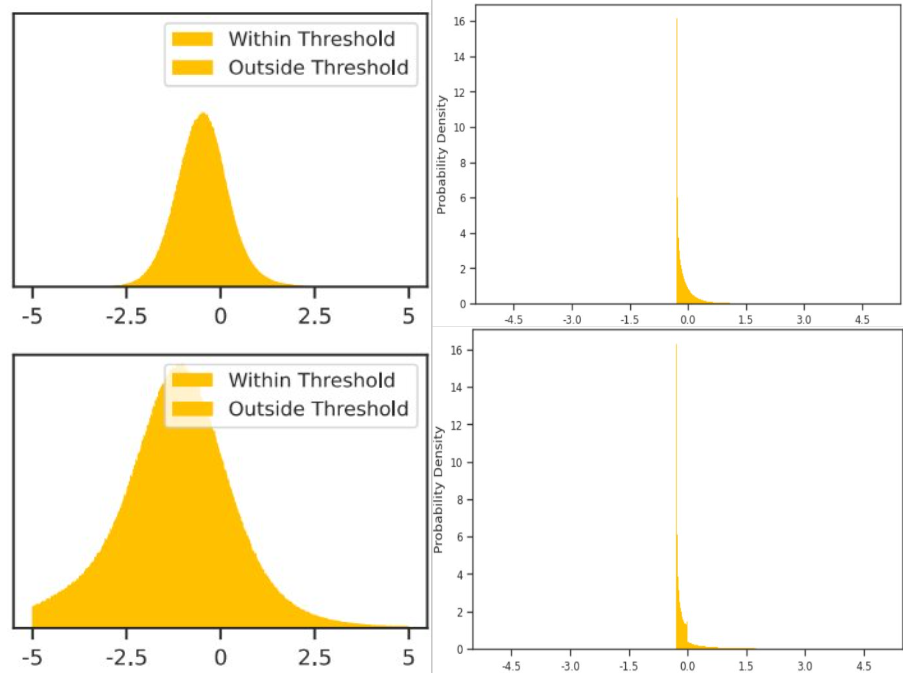}
    \hfill\includegraphics[width=0.45\linewidth]{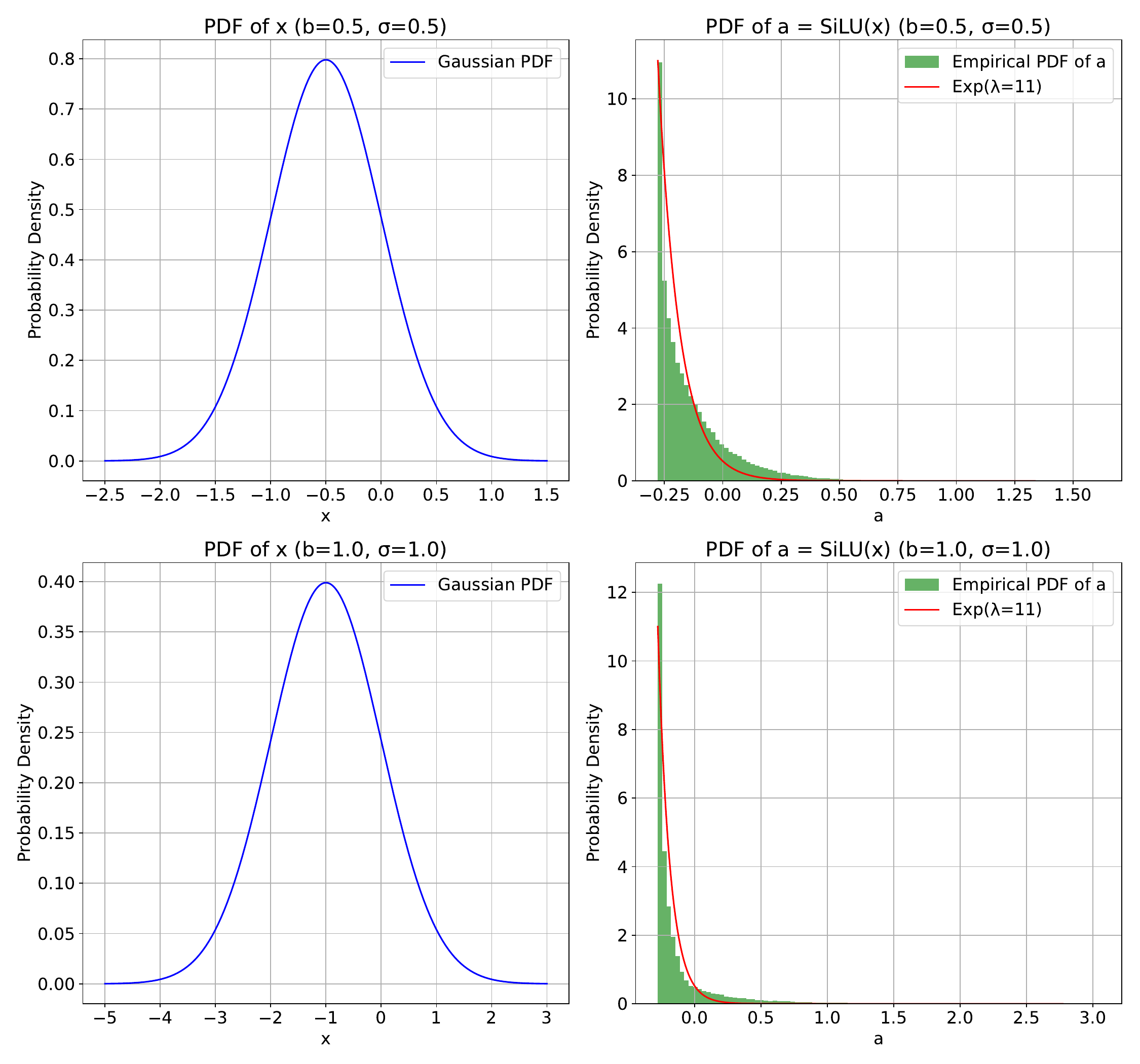}
    \caption{(\textbf{Left}) Distribution of gate-projection elements before (1st column) and after (2rd column) SiLU functions. The first row is for 15-th layer and second row is for 31-th layer. We can see that before SiLU function, the activations are roughly shifted gaussian, while after SiLU, they has very high probability density at value $x \approx -0.28$, which is the minimum value of SiLU function.
    (\textbf{Right}) Simulation of SiLU outputs on shifted gaussian variables. We find that for input $x \sim \mathcal{N}(-b, \sigma^2)$ with reasonable $b, \sigma$ (as in the left figure), the outputs after SiLU function has similar truncated unimodal distribution, and can be well fitted by shifted exponential distribution, i.e., $a = x - c$ for $x$ satisfies exponential distribution with parameters $\lambda = 11$ and shifted constant $c = 0.28$. 
    }
    \label{fig:gate_no_relu}
\end{figure}

\begin{remark}
    Here we discuss the rationality of the theorem assumptions: 
    First, the choice of $\mathcal{L}_2$ loss, independence between random variables, and gaussian assumptions are widely used in machine learning theory community~\cite{tripuraneni2020theory,tripuraneni2021provable,du2020few,thekumparampil2021sample,pmlr-v162-chen22j,wang2023improved}, and from Figure~\ref{fig:sparsity_ob} in the main paper, we can observe that the distribution of elements in activations after up projection satisfy gaussian distrbution.
    
    On the other hand, the shifted exponential distribution on gate-projection activations mainly comes from the property of SiLU function. The distribution of gate-projection elements before and after SiLU functions are shown in Figure~\ref{fig:gate_no_relu}.
    We find that the distribution of $a_{\text{gate},i}$ can be well fitted by shifted exponential distribution, i.e., $a = x - c$ for $x$ satisfies exponential distribution with parameters $\lambda \geq 10$ and shifted constant $c = 0.28$. 
    Here $c$ is the negative value of the minimum of SiLU function, and thus is a fixed value. 
    Therefore, our assumptions are consistent with experimental observations, and we can see that condition $\lambda c \geq 2$ is also satisfied.
    We add this data distribution assumption mainly because the theoretical difficulty for handling the reverse function of SiLU function, and we will see that even with this simplification, the proof is still non-trivial.
\end{remark}

Then we comes to the proof of our main theorem:

\textit{Proof of \cref{thm: formal}}:
From assumptions and~\cref{lemma: basic expectation},~\cref{lemma: mean zero}, we have
\begin{align}
    \mathcal{L}_{\textnormal{down}} 
        &= n \sigma_W^2 \cdot  \sigma_{\text{up}}^2 \cdot\mathbb{E}\bigl\|\mathbf{a}_{\textnormal{down}} - \mathtt{S}_t(\mathbf{a}_{\textnormal{down}})\bigr\|_2^2, 
        \\
    \mathcal{L}_{\textnormal{up}} 
        &= n \sigma_W^2 \cdot  \sigma_{\text{up}}^2 \cdot\mathbb{E}\bigl\|\mathbf{a}_{\textnormal{down}} - \mathbf{a}_{\textnormal{gate}} \odot \mathtt{S}_t(\mathbf{a}_{\textnormal{up}})\bigr\|_2^2,\\
    \mathcal{L}_{\textnormal{gate}} 
        &= n \sigma_W^2 \cdot  \sigma_{\text{up}}^2 \cdot\mathbb{E}\bigl\|\mathbf{a}_{\textnormal{down}} - \mathtt{S}_t(\mathbf{a}_{\textnormal{gate}}) \odot \mathbf{a}_{\textnormal{up}}\|_2^2.
    \end{align}

Note that obviously, for any vector $\mathbf{a}$ and any fixed ratio of non-sparsity rate $1-\eta$, $\mathtt{S}_t(a)$ is the sparsified vectors with maximum norm, and all three kind of sparsification strategies have the same non-sparsity ratios, so we must have
\begin{align}
    \|\mathbf{a}_{\textnormal{down}} - \mathtt{S}_t(\mathbf{a}_{\textnormal{down}})\bigr\|_2^2 &\leq \|\mathbf{a}_{\textnormal{down}} - \mathbf{a}_{\textnormal{gate}} \odot \mathtt{S}_t(\mathbf{a}_{\textnormal{up}})\bigr\|_2^2
    \\
    \|\mathbf{a}_{\textnormal{down}} - \mathtt{S}_t(\mathbf{a}_{\textnormal{down}})\bigr\|_2^2 &\leq 
    \|\mathbf{a}_{\textnormal{down}} - \mathtt{S}_t(\mathbf{a}_{\textnormal{gate}}) \odot \mathbf{a}_{\textnormal{up}}\|_2^2
\end{align}

On the other hand, note that
\begin{align}
    \mathbb{E}\bigl\|\mathbf{a}_{\textnormal{down}} - \mathbf{a}_{\textnormal{gate}} \odot \mathtt{S}_t(\mathbf{a}_{\textnormal{up}})\bigr\|_2^2
    &= \mathbb{E}\bigl\|\mathbf{a}_{\textnormal{gate}} \odot (\mathbf{a}_{\textnormal{up}} - \mathtt{S}_t(\mathbf{a}_{\textnormal{up}}))\bigr\|_2^2\\
    &= m\mathbb{E}\bigl[{a}_{\textnormal{gate},i} \cdot ({a}_{\textnormal{up},i} - \mathtt{S}_t({a}_{\textnormal{up},i}))\bigr]^2, \quad (\text{i.i.d.})\\
    &= m\mathbb{E}\bigl[{a}_{\textnormal{gate},i}^2 \bigl]\cdot \mathbb{E}\bigl[({a}_{\textnormal{up},i} - \mathtt{S}_t({a}_{\textnormal{up},i}))\bigr]^2, \quad (\text{independence, Lemma~\ref{lemma: mean zero}, Lemma~\ref{lemma: basic expectation}})
\end{align}

Similar formulas hold for $\mathcal{L}_{\text{gate}}$. Then combining~\cref{lemma: gaussian},~\cref{lemma: exp}, and~\cref{lemma: F G key}, we can get the results.

\qed

\subsection{Technical Proof}
\begin{lemma}
\label{lemma: gaussian}
Assume that random variable $ a \sim \mathcal{N}(0, \sigma^2) $. For a given $ \eta \in (0,1) $, define the threshold $ t_{\eta} $ such that $P\left( |a| > t_{\eta} \right) = \eta$. 
Then if we define the inverse sparsity function $ \bar{S}_{t_\eta}(a) $ as
\begin{equation}\label{eq:sa_def}
\bar{S}_{t_\eta}(a) = 
\begin{cases}
0, & \textnormal{if } |a| \geq t_{\eta}, \\
a, & \textnormal{otherwise}.
\end{cases}
\end{equation}
Then, the threshold $ t_{\eta} $ and the expectation $ \mathbb{E}[\bar{S}_{t_\eta}(a)^2] $ are given by
\begin{equation}\label{eq:threshold_expression}
t_{\eta} = \sigma \, \Phi^{-1}\left(1 - \frac{\eta}{2}\right),
\end{equation}
and
\begin{equation}\label{eq:expected_sa2}
\mathbb{E}[\bar{S}_{t_\eta}(a)^2] = \sigma^2 \left[ 1 - \eta - 2 z_{\eta} \phi(z_{\eta}) \right] = \Big( 1 - \eta - 2 z_{\eta} \phi(z_{\eta}) \Big) \cdot \mathbb{E}[a^2] ,
\end{equation}
where $ z_{\eta} = \Phi^{-1}\left(1 - \frac{\eta}{2}\right) =  t_\eta/\sigma $. And as defined in Appendix~\ref{sec: preliminary}, $ \phi(x) = \frac{1}{\sqrt{2\pi}} e^{-x^2 / 2} $ is the PDF of the standard normal distribution, $ \Phi^{-1}(\cdot) $ denotes its inverse cumulative distribution function (CDF).
\end{lemma}

\begin{proof}
\textbf{Threshold $ t_{\eta} $:}

Given $ a \sim \mathcal{N}(0, \sigma^2) $, standardize $ a $ by defining $ Z = \frac{a}{\sigma} $, so that $ Z \sim \mathcal{N}(0,1) $. 
Therefore:
$$
P\left( |Z| > \frac{t_{\eta}}{\sigma} \right) \leq \eta.
$$
Due to the symmetry, we have
$$
2P\left( Z > \frac{t_{\eta}}{\sigma} \right) \leq \eta \quad \Rightarrow \quad P\left( Z > \frac{t_{\eta}}{\sigma} \right) \leq \frac{\eta}{2}.
$$
Therefore we have
$$
\frac{t_{\eta}}{\sigma} = \Phi^{-1}\left(1 - \frac{\eta}{2}\right) \quad \Rightarrow t_{\eta} = \sigma \, \Phi^{-1}\left(1 - \frac{\eta}{2}\right).
$$

\textbf{Expectation $ \mathbb{E}[\bar{S}_{t_\eta}(a)^2] $:}

First note that
\begin{align}
\mathbb{E}[\bar{S}_{t_\eta}(a)^2] 
&= \mathbb{E}\left[ a^2 \cdot \mathbf{1}_{\{|a| < t_{\eta}\}} \right]\\
&= \mathbb{E}[a^2] - \mathbb{E}\left[ a^2 \cdot \mathbf{1}_{\{|a| \geq t_{\eta}\}} \right]\\
&= \sigma^2 - \mathbb{E}\left[ a^2 \cdot \mathbf{1}_{\{|a| \geq t_{\eta}\}} \right]
\end{align}

And let $z = a/\sigma \sim N(0, \sigma^2)$, $z_{\eta} = t_{\eta}/\sigma$, we can obtain
\begin{align}
    \mathbb{E}\left[ a^2 \cdot \mathbf{1}_{\{|a| \geq t_{\eta}\}} \right]
    &= 2 \sigma^2\int_{z_\eta}^{\infty}z^2 \phi(z) dz\\
    &= 2 \sigma^2 \{[-z\phi(z)]^{\infty}_{z_{\eta}} + \int_{z_\eta}^{\infty}\phi(z)dz\}\\
    &= 2 \sigma^2 (z_\eta \phi(z_\eta) + Q(z_\eta))
\end{align}

where $ Q(z) = 1 - \Phi(z) $. Substituting $ z_{\eta} = \Phi^{-1}\left(1 - \frac{\eta}{2}\right) $, we get $Q(z_{\eta}) = \frac{\eta}{2}$. Therefore, finally we have
\begin{equation}
\mathbb{E}[\bar{S}_{t_\eta}(a)^2] = \sigma^2 - 2 \sigma^2 \left[ \frac{z_{\eta}}{\sqrt{2\pi}} e^{-z_{\eta}^2 / 2} + \frac{\eta}{2} \right] = \sigma^2 \left[ 1 - \eta - 2 z_{\eta} \phi(z_{\eta}) \right]
\end{equation}
This concludes the proof.
\end{proof}

\begin{lemma}
\label{lemma: exp}
We define $t_\eta$ and $\bar{S}_{t_\eta}$ similarly  as Lemma~\ref{lemma: gaussian}.
If $x$ satisfies exponential distribution with parameter $\lambda$, and $a = x - c$ for some constant $c$ ($c \geq t_\eta$). Then we have
    \begin{equation}
        \mathbb{E}[a^2] = \frac{2}{\lambda^2} - \frac{2c}{\lambda} + c^2
    \end{equation}
    And,
    \begin{equation}
        \mathbb{E}[\bar{S}_{t_\eta}(a)^2] = 
        e^{\lambda (t_\eta-c)} \big(  
\frac{2}{\lambda^2} 
- 2 \frac{t_\eta}{\lambda}
+ t_\eta^2 
\big) - e^{-\lambda (c + t_\eta)} \big( 
 \frac{2}{\lambda^2} 
 + 2 \frac{t_\eta}{\lambda}
+ t_\eta^2
\big)
\end{equation}
Furthermore,  $t_\eta$ satisfies
\begin{equation}
    t_\eta  =
    \begin{cases} 
    \frac{1}{\lambda}\sinh^{-1}\Big(\frac{1-\eta}{2}e^{\lambda c}\Big), & \eta \geq \exp(-2\lambda c) \\
    - \frac{1}{\lambda}\ln(\eta) - c, &\text{otherwise}. \\
    \end{cases}
\end{equation}

where $\sinh^{-1}(x) = \ln(x + \sqrt{x^2 + 1})$.
\end{lemma}

\begin{proof}

\textbf{Expectation} $ \mathbb{E}[a^2] $:

Since $ x \sim \textnormal{Exp}(\lambda) $, we have $ \mathbb{E}[x] = \frac{1}{\lambda} $ and $ \textnormal{Var}(x) = \frac{1}{\lambda^2} $. For $ a = x - c $:
\begin{equation}
\mathbb{E}[a] = \mathbb{E}[x] - c = \frac{1}{\lambda} - c, \quad \textnormal{Var}(a) = \textnormal{Var}(x) = \frac{1}{\lambda^2}.
\end{equation}
Thus:
\begin{equation}
\mathbb{E}[a^2] = \textnormal{Var}(a) + (\mathbb{E}[a])^2 = \frac{1}{\lambda^2} + \left(\frac{1}{\lambda} - c\right)^2 = \frac{2}{\lambda^2} - \frac{2c}{\lambda} + c^2.
\end{equation}

\textbf{Expectation} $ \mathbb{E}[\bar{S}_{t_\eta}(a)^2] $:

By definition:
\begin{equation}
\mathbb{E}[\bar{S}_{t_\eta}(a)^2] = \mathbb{E}[a^2 \cdot \mathbf{1}_{\{|a| < t_\eta\}}] = \int_{c - t_\eta}^{c + t_\eta} (x - c)^2 \lambda e^{-\lambda x} dx.
\end{equation}
where $ a = x - c $. Recover it to the integral on $a$:
\begin{equation}
\mathbb{E}[\bar{S}_{t_\eta}(a)^2] = \lambda e^{-\lambda c} \int_{-t_\eta}^{t_\eta} a^2 e^{-\lambda a} da.
\end{equation}
Exploiting symmetry (valid if $ c \geq t_\eta $):
\begin{align}
\int_{-t_\eta}^{t_\eta} a^2 e^{-\lambda a} da 
&= \int_{-t_\eta}^{0} a^2 e^{-\lambda a} da + \int_{0}^{t_\eta} a^2 e^{-\lambda a} da \\
&= 2 \int_{0}^{t_\eta} a^2 (\frac{e^{\lambda a} + e^{-\lambda a}}{2}) da \\
&= 2 \int_{0}^{t_\eta} a^2 \cosh(\lambda a) da,
\end{align}

where $ \cosh(\lambda a) = \frac{e^{\lambda a} + e^{-\lambda a}}{2} $. Using the integral formula~\cite{gradshteyn2014table}:
\begin{equation}
\int a^2 \cosh(\lambda a) da = \frac{e^{\lambda a}}{2\lambda^3} \left( (\lambda a)^2 - 2\lambda a + 2 \right) 
- \frac{e^{-\lambda a}}{2\lambda^3} \left( (\lambda a)^2 + 2\lambda a + 2 \right) + C,
\end{equation}
Then we have
\begin{align}
\mathbb{E}[\bar{S}_{t_\eta}(a)^2] 
&= 
2 \lambda e^{-\lambda c}\int_{0}^{t_\eta} a^2 \cosh(\lambda a) da \\
&= 
\frac{ e^{-\lambda c}}{\lambda^2} \Big[ e^{\lambda t_\eta} \big( (\lambda t_\eta)^2 - 2\lambda t_\eta + 2 \big) - e^{-\lambda t_\eta} \big( (\lambda t_\eta)^2 + 2\lambda t_\eta + 2 \big) \Big].\\
&= 
e^{\lambda (t_\eta-c)} \big(  
\frac{2}{\lambda^2} 
- 2 \frac{t_\eta}{\lambda}
+ t_\eta^2 
\big) - e^{-\lambda (c + t_\eta)} \big( 
 \frac{2}{\lambda^2} 
 + 2 \frac{t_\eta}{\lambda}
+ t_\eta^2
\big)
\end{align}
\textbf{Determination of $t_\eta$:}

The threshold satisfies $P(|a| \geq t_\eta) = \eta$. For $a = x - c$ with $x \sim \text{Exp}(\lambda)$:
\begin{equation}
P(x \geq c + t_\eta) + P(x \leq c - t_\eta) = e^{-\lambda(c + t_\eta)} + \left(1 - e^{-\lambda(c - t_\eta)}\right) = \eta.
\end{equation}

\textbf{Case 1: $\eta \leq e^{-2\lambda c}$}

When $t_\eta > c$, the lower tail vanishes:
\begin{equation}
P(x \geq c + t_\eta) = e^{-\lambda(c + t_\eta)} = \eta \implies t_\eta = -\frac{1}{\lambda}\ln(\eta) - c. 
\end{equation}
It's clear the condition in this case is: $t_\eta \geq c \Rightarrow \eta \leq \exp(-2\lambda c)$.

\textbf{Case 2: $\eta \geq e^{-2\lambda c}$}

When $t_\eta \leq c$, both terms contribute:
\begin{align}
e^{-\lambda c}(e^{-\lambda t_\eta} - e^{\lambda t_\eta}) + 1 &= \eta, \\
e^{\lambda t_\eta} - e^{-\lambda t_\eta} &= (1 - \eta)e^{\lambda c}.
\end{align}
Note that $\sinh(y) = \frac{e^y - e^{-y}}{2}$:
\begin{equation}
\sinh(\lambda t_\eta) = \frac{(1 - \eta)e^{\lambda c}}{2} \implies t_\eta = \frac{1}{\lambda}\sinh^{-1}\left(\frac{1 - \eta}{2}e^{\lambda c}\right).
\end{equation}

It's easy to check now we have $\eta \geq \exp(-2\lambda c)$.
This concludes the proof.
\end{proof} 

\begin{remark}\label{remark: exp}
From calculation, we can have some approximation of the key terms in \cref{lemma: exp}:
    If $\lambda c \gg 1$ and $\eta \in [\exp(-2\lambda c), 1]$, then we have
    \begin{equation}
        t_\eta \approx c + \frac{1}{\lambda}\ln(1-\eta)
    \end{equation}
    \begin{equation}
        \mathbb{E}[\bar{S}_{t_\eta}(a)^2] \approx
  (1-\eta)
  \Bigl(\frac{2}{\lambda^2}
    - 2 \frac{t_\eta}{\lambda}
    + t_\eta^2
  \Bigr)
    \end{equation}
Then if $\eta \in [\exp(-2\lambda c), 1/2]$, we can see $t_\eta$ is very close to $c$, which matches our experiment observations in \cref{fig:gate_no_relu}.
\end{remark}

Then we compare the second moments calculated above.

\begin{lemma}\label{lemma:qgh}
    Let $q_\eta = \frac{1}{p}\sinh^{-1}(\frac{1-\eta}{2}e^{p})$, $g_\eta = p(q_\eta - 1)$, and $h_\eta = p(q_\eta + 1)$. Then if $p \geq 2$ and $\eta \in [ e^{-2p}, 1/2]$, we have
    \begin{equation}
        0 < 1 + \frac{\ln(1-\eta)}{p} < q_\eta < 1
    \end{equation}
    \begin{equation}
        \ln(1-\eta) < g_\eta < 0, \qquad 2p + \ln(1-\eta) < h_\eta < 2p
    \end{equation}
\end{lemma}

\begin{proof}
    Note that 
    \begin{equation}\sinh^{-1}(x) = \ln(x + \sqrt{x^2 + 1}) \geq \ln(2x)
    \end{equation}
    So we have
    \begin{align}
        q_\eta 
        = \frac{1}{p}\sinh^{-1}\left(\frac{1-\eta}{2}e^p\right) 
        > \frac{1}{p}\ln\left(2 \cdot \frac{1-\eta}{2}e^p\right)
        = \frac{\ln(1-\eta)}{p} + 1 > 0.
    \end{align}
    And the last inequality holds when $\eta \leq 0.5 < 1 - e^{-2} \leq 1 - e^{-p}$.
    For the upper bound, from  $\eta \geq e^{-2p}$, we have:
    \begin{equation}
        \frac{1-\eta}{2}e^p \leq \frac{1 - e^{-2p}}{2}e^p = \frac{e^p - e^{-p}}{2} = \sinh(p).
    \end{equation}
    Since $\sinh^{-1}$ is strictly increasing and $\sinh^{-1}(\sinh(p)) = p$, we have:
    \begin{equation}
        q_\eta = \frac{1}{p}\sinh^{-1}(\frac{1-\eta}{2}e^p) < 1.
    \end{equation}
    And thus inequalities for $g_\eta$ also hold.
\end{proof}

\begin{lemma}\label{lemma:fp}
For $p > 0$, the function
\begin{equation}
    f(p) = e^{-2p} \cdot \frac{2 + 2p + p^2}{2 - 2p + p^2}
\end{equation}
is strictly decreasing. And therefore when $p \geq 2$, $f(p) \leq f(2) = 5e^{-4}$.
\end{lemma}

\begin{proof}
Let $N(p) = p^2 + 2p + 2$ and $D(p) = p^2 - 2p + 2$. The derivative of $f(p)$ is:
\begin{align}
f'(p) 
&= e^{-2p} \left( \frac{N'(p)D(p) - N(p)D'(p) - 2 N(p)D(p)}{D(p)^2} \right)\\
&= e^{-2p} \left( \frac{(2p + 2)(p^2 - 2p + 2) - (p^2 + 2p + 2)(2p - 2) - 2(p^2 - 2p + 2)(p^2 + 2p + 2)}{D(p)^2} \right) \\
&= e^{-2p} \cdot \frac{-4p^2 + 8 - 2(p^4 + 4)}{D(p)^2}  \\
&= e^{-2p} \cdot \frac{-2p^4 - 4p^2}{D(p)^2}\\
&<0.
\end{align}

Therefore, $f(p)$ is strictly decreasing.
\end{proof}

\begin{lemma}\label{lemma: F G key}
    We define
    \begin{equation}
        F(\eta) = 1 - \eta - 2 z_{\eta} \phi(z_{\eta}) 
    \end{equation}
    where $z_{\eta} = \Phi^{-1}\left(1 - \frac{\eta}{2}\right)$ and $ \phi(x) = \frac{1}{\sqrt{2\pi}} e^{-x^2 / 2}$. And we define
    \begin{equation}
        G(\eta, p) = 
e^{p(q_\eta-1)} \big(\frac{  
2/p^2
- 2 q_\eta/p
+ q_\eta^2 
}{2/p^2 - 2/p + 1}\big) - e^{-p (1 + q_\eta)} \big( \frac{
 2/p^2 
 + 2 q_\eta/p
+ q_\eta^2}{2/p^2 - 2/p + 1}
\big)
    \end{equation}
    where $q_\eta = \frac{1}{p}\sinh^{-1}(\frac{1-\eta}{2}e^{p})$. 
    Then if $p \geq 2$ and $\eta \in [e^{-4}, 0.5]$, we have $ F (\eta) < G(\eta, p)$.
\end{lemma}

\begin{figure}
    \centering
    \includegraphics[width=0.7\linewidth]{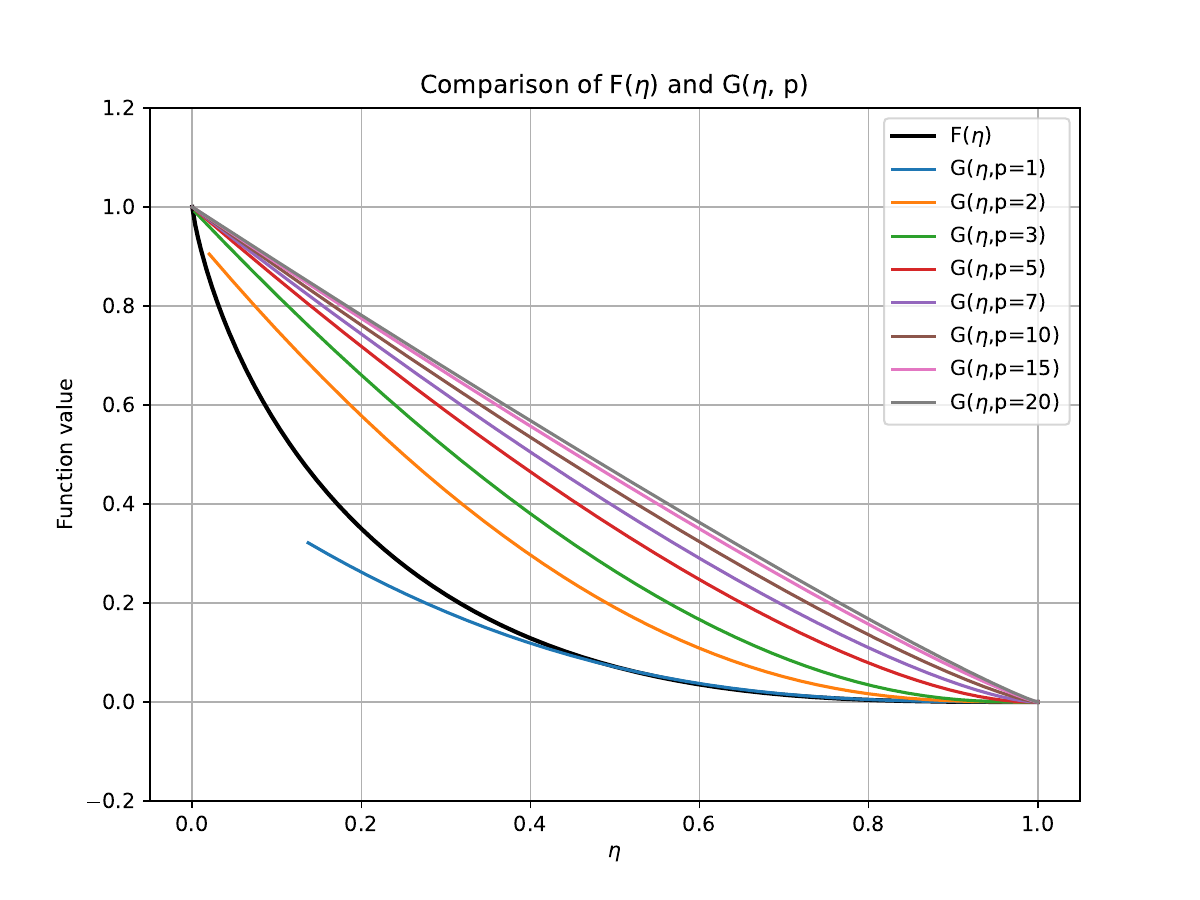}
    \caption{Comparison between $F(\eta)$ and $G(\eta,p)$.
    When $p\geq2$ and $\eta \in [e^{-4}, 0.5]$, we can see that $G(\eta, p) \geq F(\eta)$. And actually this gap increase as $p$ increases.}
    \label{fig:G F}
\end{figure}

\begin{proof}
    Let $g_\eta = p(q_\eta - 1)$ and $h_\eta = p(q_\eta + 1)$. Then for the first term we have
    \begin{align}
        G_1(\eta, p)
        &= e^{p(q_\eta-1)} \big(\frac{  
2/p^2
- 2 q_\eta/p
+ q_\eta^2 
}{2/p^2 - 2/p + 1}\big)\\
&= e^{g_\eta} (\frac{2 - 2pq_\eta + q_\eta^2p^2}{2 - 2p + p^2})\\
&= e^{g_\eta} (1 + \frac{- p^2 + 2p - 2pq_\eta + q_\eta^2p^2}{(p-1)^2 + 1})\\
&= e^{g_\eta} (1 + \frac{p[-p + 2 - 2q_\eta + q_\eta^2p]}{(p-1)^2 + 1})\\
&= e^{g_\eta} (1 + \frac{p[p(q_\eta^2-1) + 2(1-q_\eta)]}{(p-1)^2 + 1}) \\
&= e^{g_\eta} (1 + \frac{p\cdot(q_\eta-1)\cdot(p(q_\eta + 1) - 2)}{(p-1)^2 + 1}) \\
&= e^{g_\eta} (1 + \frac{g_\eta(h_\eta - 2)}{(p-1)^2 + 1})
    \end{align}

From Lemma~\ref{lemma:qgh}, 
it's easy to see that
\begin{align}
    G_1(\eta, p) &> (1-\eta) \cdot (1 + \ln(1-\eta)\cdot\frac{ (2p-2)}{(p-1)^2 + 1}) \\
    &> 1-\eta + (1-\eta)\ln(1-\eta)
\end{align}
Here we use $\ln(1-\eta) < 0$ and the property that $f(p):=\frac{ 2p-2}{(p-1)^2 + 1} \leq f(2) = 1$. 

Similarly, for the second term of $G$, we have
\begin{align}
    G_2(\eta, p) 
    &=e^{-p (1 + q_\eta)} \big( \frac{
 2/p^2 
 + 2 q_\eta/p
+ q_\eta^2}{2/p^2 - 2/p + 1}
\big) \\
&= e^{-h_\eta}(\frac{2 + 2pq_\eta + q_\eta^2p^2}{2 - 2p + p^2})\\
&< (1-\eta)\cdot e^{-2p}(\frac{2 + 2p + p^2}{2 - 2p + p^2}) \quad (\text{Lemma~\ref{lemma:qgh}})\\
&< 5e^{-4}(1-\eta) \quad (\text{Lemma~\ref{lemma:fp}})
\end{align}

Therefore we have
\begin{align}
    G(\eta, p) - F(\eta) 
    &= G_1(\eta, p) - G_2(\eta, p) - F(\eta) \\
    &> 1 - \eta + (1-\eta)\ln(1-\eta) - 5e^{-4}(1-\eta) - [1 - \eta - 2z_\eta\phi(z_\eta)]\\
    &= 2z_\eta\phi(z_\eta) - (1-\eta)[5e^{-4} - \ln(1-\eta)] \\
    &= Q(\eta)\label{eq: G minus F}
\end{align}

where $z_{\eta} = \Phi^{-1}\left(1 - \frac{\eta}{2}\right)$ and $ \phi(x) = \frac{1}{\sqrt{2\pi}} e^{-x^2 / 2}$.

Note that for $m(x) = x\phi(x)$, $m'(x) = (1-x^2)\phi(x)$, and thus $m(x)$ is strictly decreasing on $x \in (-\infty, -1) \cup (1,\infty)$ and strictly increasing on $(-1,1)$. 
And further note that $\Phi^{-1}(\cdot)$ is strictly decreasing and continous, 
then since $\eta \in [e^{-2p}, 0.5]$, $z_\eta \in [\Phi^{-1}(0.75), \Phi^{-1}(1-e^{-2p}/2)]$, where $\Phi^{-1}(0.75) \approx 0.674$ and $\Phi^{-1}(1-e^{-2p}/2) > \Phi^{-1}(0.975) \approx 1.960 > 1$. Therefore, $m_2(\eta) = 2m(z_\eta)$ first increase, and then decrease on $\eta \in [e^{-2p}, 0.5] \supset [e^{-4}, 0.5]$.

Similarly, for $n(\eta) = (1-\eta)[5e^{-4} - \ln(1-\eta)]$, $n'(\eta) = 1-5e^{-4} + \ln (1-\eta)$. We can see $n'(\eta)$ consistently increase for $\eta \in[0, \eta^*] \supset [e^{-2p}, 0.5] \supset [e^{-4}, 0.5]$, where $\eta^* = 1-\exp( 5e^{-4} - 1) \approx 0.632$.

Combining these, to prove that $Q(\eta)$ (\cref{eq: G minus F}) is positive on $\eta \in [e^{-4}, 0.5]$, we just need to make sure $m_2(e^{-4}) > n(e^{-4})$ and $m_2(0.5) > n(0.5)$. And from calculation we get
\begin{equation}
    m_2(e^{-4})\approx 0.116 > n(e^{-4}) \approx 0.108, \qquad 
    m_2(0.5)\approx 0.429 > n(0.5) \approx 0.392
\end{equation}

Therefore, we have $G(\eta, p) - F(\eta) > Q(\eta) > 0$ for all $p \geq 2$ and $\eta \in [e^{-4}, 0.5]$. This completes the proof.

\end{proof}

\begin{remark}
    Lemma~\ref{lemma: F G key} is the key lemma of the whole proof, which compares the variance caused by the inverse sparsity function of up and gate function. We can further visualize $F(\eta)$ and $G(\eta, p)$ in~\cref{fig:G F}. We see that the visualization results match our proof, and showing that larger $p$ will has larger $G(\eta, p) - F(\eta)$ values.
\end{remark}

\begin{lemma}\label{lemma: basic expectation}
    Assume that $x \in \mathbb{R}^{m}$ and $W \in \mathbb{R}^{m \times n}$ are random vector and matrix whose elements are independent to each other. And all the $W_{ij}$ in $W$ satisfies $W_{ij} \sim \mathcal{N}(0, \sigma^2)(i \in [m], j \in [n])$. Then we have
    \begin{equation}
        \mathbb{E} [\| x W \|_2^2] = n \sigma^2 \mathbb{E} [\|x\|_2^2]. 
    \end{equation}
    Furthermore, if all $x_i$ are \textit{i.i.d.} (independent and identically distributed), with mean $0$ and variance $\sigma_x^2$, then 
    \begin{equation}
        \mathbb{E} [\| x W \|_2^2] = n m \sigma^2 \cdot \sigma_x^2. 
    \end{equation}
\end{lemma}
\begin{proof}
Let ${W}_j$ denote the $j$-th column of matrix $W$. Due to independence:
\begin{equation}
\mathbb{E}[\|xW\|_2^2] = \mathbb{E}\left[\sum_{j=1}^n (x{W}_j)^2\right] \label{eq:norm_expansion} = \sum_{j=1}^n \mathbb{E}\left[(x{W}_j)^2\right] = n \mathbb{E}\left[(x{W}_j)^2\right]
\end{equation}
Similarly by independence between $x$ and $W$.
\begin{align}
\mathbb{E}\left[(x{W}_j)^2\right] &= \mathbb{E}\left[\left(\sum_{i=1}^m x_i W_{ij}\right)^2\right]  \\
&= \sum_{i=1}^m \mathbb{E}[x_i^2]\mathbb{E}[W_{ij}^2] + 2\sum_{i<k}\mathbb{E}[x_ix_k]\mathbb{E}[W_{ij}W_{kj}]  \\
&= \sum_{i=1}^m \mathbb{E}[x_i^2]\sigma^2 + 0  \\
&= \sigma^2 \mathbb{E}\left[\sum_{i=1}^m x_i^2\right] \label{eq:sum_variance} \\
&= \sigma^2 \mathbb{E}[\|x\|_2^2] 
\end{align}
Combine them and we get the first equation. And the second equation is totally similar.
\end{proof}

\begin{lemma}\label{lemma: mean zero}
      If $a \sim \mathcal{N}(0, \sigma_a^2)$, $b$ is a random variable independent to $a$, sparsity function $\mathtt{S}_t$ is defined as \cref{eq: s_t def}. 
Then for any $t > 0$, $\mathbb{E}[(a - \mathtt{S}_t(a)) \cdot b] = \mathbb{E}[a \cdot (b - \mathtt{S}_t(b))] = 0$.
\end{lemma}

\begin{proof}
    Just need to note that $a$ and $b$ are independent and from symmetry, $\mathbb{E}[\mathtt{S}_t(a)] = 0$, and then the proof is trivial. 
\end{proof}

\newpage
\section{Sparsity Insensitivity of the Up Projection in Dense LLMs}
We evaluate the sparsity sensitivity of the up projection on LLaMA-3-8B~\citep{meta2024llama3}, with results consistent with our findings on MoE models.Some results are presented in~\cref{tab:dense}.
\begin{table*}[ht]
\begin{center}
\caption{Sparsity Insensitivity of the Up Projection in Dense LLMs}
\label{tab:dense}
\begin{small}
\begin{sc}
\begin{tabular}{@{}lccccccc@{}}
\toprule
Method     & bool\_q & sci\_q & openbookqa & winogrande & arc\_challenge & arc\_easy & average \\
\midrule
base       & 0.8187  & 0.961  & 0.370      & 0.7348     & 0.5026         & 0.8085    & 0.6993  \\
up-90\%    & 0.7116  & 0.925  & 0.300      & 0.6717     & 0.4002         & 0.7066    & 0.6192  \\
down-90\%  & 0.7780  & 0.959  & 0.336      & 0.6922     & 0.4241         & 0.7504    & 0.6566  \\
\bottomrule
\end{tabular}
\end{sc}
\end{small}
\end{center}
\end{table*}

\section{Downstream Tasks 
 Performance Details}
\label{sec:downstream_task}

We present the detailed evaluation results of \model{} and baseline methods across downstream tasks in ~\cref{exp:accuracy} in~\cref{tab:downstream_all}.
We also evaluate the impact of different projection matrix sparsification sensitivity on downstream tasks in ~\cref{tab:SparsificationSensitivity}.

\begin{table*}[htp!]
\begin{center}
\caption{Performance of Downstream Tasks under Different Compression Methods}
\label{tab:downstream_all}
\begin{small}
\begin{sc}
\begin{tabular}{@{}lcccccccc@{}}
\toprule 
                       & MMLU@5 & BoolQ         & SCIQ         & QA     & WG      & Arc-C          & Arc-E          & Average \\

\midrule
Mixtral-8*7B & 0.695 & 0.853 & 0.968 & 0.354 & 0.762 & 0.567 & 0.843 & 0.720 \\
\midrule
HQQ int3               & 0.608 & 0.809  & 0.955  & 0.292      & 0.712     & 0.481 & 0.800  & 0.665  \\
CATS-80\%                   & 0.617 & 0.792   & 0.903  & 0.322      & 0.670       & 0.515 & 0.782  & 0.657  \\
Chess-80\%                  & 0.612 & 0.802  & 0.912  & 0.302      & 0.694     & 0.498 & 0.781  & 0.657  \\
\model{}-$\mathbf{W}^{\textnormal{up}}$-80\%                  & 0.654 & 0.829  & 0.944  & 0.344      & 0.732     & 0.532 & 0.816  & 0.693  \\
\midrule
HQQ int2               & 0.234 & 0.485   & 0.331  & 0.144      & 0.493     & 0.212 & 0.279  & 0.311  \\
CATS-90\%                   & 0.377 & 0.704  & 0.826  & 0.272      & 0.586     & 0.442 & 0.709  & 0.559  \\
Chess-90\%                  & 0.424 & 0.727  & 0.839  & 0.278      & 0.604     & 0.410 & 0.694   & 0.568  \\
\model{}-$\mathbf{W}^{\textnormal{up}}$-90\%                 & 0.601 & 0.787  & 0.933  & 0.312      & 0.670     & 0.497 & 0.788  & 0.656  \\
\midrule
\model{}-80\% & 0.605 & 0.810 & 0.951  & 0.336      & 0.717     & 0.509 & 0.803   & 0.676  \\
\model{}-90\% & 0.531 & 0.835 & 0.952  & 0.276      & 0.695     & 0.458 & 0.762 & 0.644  \\
\bottomrule
\end{tabular}%
\end{sc}
\end{small}
\end{center}
\end{table*}


\begin{table}[htp!]
\begin{center}
\caption{Performance of Downstream Tasks Under Different Sparse Strategies}
\label{tab:downstream_sparse}
\begin{small}
\begin{sc}
\begin{tabular}{@{}lcccccc@{}}
\toprule
     & \texttt{0\%}                       & \texttt{50\%}   & \texttt{60\%}   & \texttt{70\%}   & \texttt{80\%}   & \texttt{90\%}   \\
    \midrule
gate & \multirow{3}{*}{0.7247} & 0.7228 & 0.7140 & 0.7035 & 0.6640 & 0.5897 \\
up   &                         & 0.7199 & 0.7148 & 0.7038 & 0.6971 & 0.6646 \\
down &                         & 0.7233 & 0.7210 & 0.7201 & 0.7194 & 0.7054 \\
\bottomrule
\end{tabular}%
\end{sc}
\end{small}
\end{center}
\end{table}

\clearpage
\section{Sparsification Insensitivity of the Up Projection in More MoE models}
\label{sec:sparAppen}
Due to the smaller hidden layer dimensions of DeepSeek V2's experts, the sparsity rates tested were correspondingly lower, with results in~\cref{tab:SparsificationSensitivity-2}.

\begin{table*}[htp!]
\caption{Sparsification Sensitivity of MoE Models}
\label{tab:SparsificationSensitivity}
\begin{center}
\begin{small}
\begin{sc}
\begin{tabular}{@{}llccccc@{}}
\toprule
\textbf{Model} & \textbf{Operation} & \texttt{50\%} & \texttt{60\%} & \texttt{70\%} & \texttt{80\%} & \texttt{90\%} \\
\midrule
\multirow{3}{*}{Mixtral-8$\times$7B} & gate & 5.8151 & 6.3379 & 7.2570 & 9.1439 & 18.5280 \\
 & down & 5.1583 & 5.2101 & 5.3252 & 5.6147 & 6.5511 \\
 & up   & 5.3164 & 5.5390 & 5.9795 & 6.9141 & 9.1250 \\
\midrule
\multirow{3}{*}{Phi-3.5-MoE-instruct} & gate & 5.4386 & 5.7809 & 6.4006 & 7.6114 & 11.2538 \\
 & down & 5.1495 & 5.2051 & 5.3255 & 5.6377 & 6.6271 \\
 & up   & 5.4092 & 5.6855 & 6.2642 & 7.3284 & 10.2146 \\
\bottomrule
\end{tabular}
\end{sc}
\end{small}
\end{center}
\end{table*}

\begin{table}[h]
\caption{Sparsification Sensitivity of DeepSeek-V2}
\label{tab:SparsificationSensitivity-2}
\begin{center}
\begin{small}
\begin{sc}
\begin{tabular}{@{}llccccc@{}}
\toprule
\textbf{Model} & \textbf{Operation} & \texttt{30\%} & \texttt{50\%} & \texttt{70\%} \\
\midrule
\multirow{3}{*}{DeepSeek-V2} & gate & 8.6434 & 8.8331 & 9.8083 &  \\
 & down & 8.6264  & 8.6400 & 8.7223  \\
 & up   & 8.6456 & 8.7818 & 9.2083  \\
\bottomrule
\end{tabular}
\end{sc}
\end{small}
\end{center}
\end{table}


\section{Quantization Insensitivity of the Up Projection in More MoE models}
\label{sec:quanAppen}
Besides Mixtral 8$\times$7B, we also evaluated the quantization insensitivity of the up-projection in Phi-3.5-MoE-instruct~\cite{abdin2024phi3technicalreporthighly}, DeepSeek-MoE-16B-Base~\citep{dai2024deepseekmoe}, and Qwen1.5-MoE-A2.7B~\citep{qwen_moe}, all of which show that the up-projection is the least sensitive to ultra-low-bit quantization.The results are shown in~\cref{tab:QuantizationSensitivity}.
\begin{table*}[htp!]
\caption{Quantization Sensitivity of MoE Models}
\label{tab:QuantizationSensitivity}
\begin{center}
\begin{small}
\begin{sc}
\begin{tabular}{@{}llccccc@{}}
\toprule
\textbf{Model} & \textbf{Operation} & \texttt{INT8} & \texttt{INT4} & \texttt{INT3} & \texttt{INT2} & \texttt{INT1} \\
\midrule
\multirow{3}{*}{Phi-3.5-MoE-instruct} & gate & 5.768 & 5.785 & 5.952 & 6.623 & 608.7 \\
 & down & 5.768 & 5.772 & 6.067 & 7.733 & 365.9 \\
 & up   & 5.769 & 5.788 & 5.899 & 6.599 & 209.3 \\
\midrule
\multirow{3}{*}{DeepSeek-MOE-16\textbf{}} & gate & 8.476 & 8.497 & 8.558 & 9.020 & 112.8 \\
 & down & 8.476 & 8.659 & 9.364 & 27.70 & 350.1 \\
 & up   & 8.476 & 8.489 & 8.602 & 9.090 & 83.57 \\
\midrule
\multirow{3}{*}{Mixtral-8$\times$7B} & gate & 5.119 & 5.158 & 5.310 & 6.245 & 1130 \\
 & down & 5.121 & 5.270 & 5.968 & 14.36 & 1910 \\
 & up   & 5.119 & 5.151 & 5.281 & 6.177 & 520.1 \\
\midrule
\multirow{3}{*}{qwen-1.5-A2.7B} & gate & 9.227 & 9.258 & 9.364 & 10.72 & 102.0 \\
 & down & 9.224 & 9.419 & 9.655 & 12.53 & 138.4 \\
 & up   & 9.226 & 9.270 & 9.426 & 10.19 & 71.06 \\
\bottomrule
\end{tabular}
\end{sc}
\end{small}
\end{center}
\end{table*}

\section{Memory Footprint of Learning-based Contextual Sparsity Predictors}
\label{sec:cost_learningpredictor}
While some existing approaches rely on learning-based prediction methods \citep{liu2023deja, shin2024sparseinfer}, these methods often incur significant memory costs. For instance, in the Mixtral-8$\times$7B model, where the hidden state dimension $d$ is 4096 and the gating weight matrix $W_{gate}$ in an MLP block has a size of $d \times k = 4096 \times 14336$, \emph{PowerInfer} \citep{xue2024powerinfer2} requires $(4096 \times 1024 + 1024 \times 14336) \times 2 \textnormal{ (bytes)} \times 256 = 9\textnormal{GB}$ of memory when the rank of the \emph{DEJAVU} predictor \citep{liu2023deja} is set to 1024. Similarly, although \emph{SparseInfer} \citep{shin2024sparseinfer} achieves a more memory-efficient design by storing only the sign bit of each element, compactly packed into 32-bit variables, it still incurs a memory footprint of $14336 \times 160 \times 4 \textnormal{ (bytes)} \times 256 = 2.19\textnormal{GB}$.

\end{document}